%% file: main.tex
\documentclass[11pt]{article}
\usepackage{fullpage}
\usepackage{amsmath,amssymb,amsthm,amsfonts,dsfont,mathtools,latexsym}

\usepackage[colorlinks = true,
            linkcolor = blue,
            urlcolor  = blue,
            citecolor = blue,
            anchorcolor = blue]{hyperref}

\usepackage{algorithm}
\usepackage[noend]{algpseudocode}
\usepackage{xcolor}
\usepackage{comment}

\newcommand{\poly}{\mathrm{poly}}
\newcommand{\sta}{\mathsf{st}}

\def\cA{\mathcal{A}}

\def\cE{\mathcal{E}}
\def\cS{\mathcal{S}}
\def\cT{\mathcal{T}}

\def\cF{\mathcal{F}}

\def\mb{\overline{m}}

\newcommand{\floor}[1]{\lfloor #1 \rfloor}
\newcommand{\ceil}[1]{\lceil #1 \rceil}
\newcommand{\abs}[1]{\left|#1\right|}
\newcommand{\expect}{\mathbb{E}}

\newcommand{\indict}{\mathbb{I}}
\newcommand{\states}{\mathcal{S}}
\newcommand{\trans}{P}
\newcommand{\actions}{\mathcal{A}}
\newcommand{\mdp}{M}
\newcommand{\mc}{C}

\newcommand{\qtile}{\mathcal{Q}}
\usepackage{algorithm}
\usepackage[noend]{algpseudocode}

\newtheorem{fact}{Fact}[section]
\newtheorem{theorem}{Theorem}[section]
\newtheorem{lemma}[theorem]{Lemma}
\newtheorem{corollary}[theorem]{Corollary}
\newtheorem{assumption}{Assumption}[section]

\theoremstyle{definition}

\title{Settling the Horizon-Dependence of Sample Complexity in Reinforcement Learning}

\date{}
\author{Yuanzhi Li \\ Carnegie Mellon University \\ \texttt{yuanzhil@andrew.cmu.edu} \and Ruosong Wang \\ Carnegie Mellon University \\ \texttt{ruosongw@andrew.cmu.edu} \and Lin F. Yang \\ University of California, Los Angles \\ \texttt{linyang@ee.ucla.edu}}

\begin{document}
\begin{titlepage}
\maketitle
\thispagestyle{empty}
\input{abs}
\end{titlepage}
\input{intro}

\input{prelim}

\input{tech.tex}

\input{discount.tex}
\section{Algorithm in the RL Setting}
In this section, we present our algorithm in the RL setting together with its analysis.
Our algorithm is divided into two parts.
In Section~\ref{sec:collect}, we first present the algorithm for collecting samples together with its analysis.
In Section~\ref{sec:perturbation}, we establish a perturbation analysis on the value functions which is crucial for the analysis in later proofs.
Finally, in Section~\ref{sec:planning}, we present the algorithm for finding near-optimal policies based on the dataset found by the algorithm in Section~\ref{sec:collect}, together with its analysis based on the machinery developed in Section~\ref{sec:perturbation}.
\input{collect.tex}

\input{perturbation.tex}
\input{planning.tex}

\input{gen_model.tex}
\section*{Acknowledgements}
The authors would like to thank Simon S. Du for helpful discussions, and to the anonymous reviewers for helpful comments.

Ruosong Wang was supported in part by NSF IIS1763562, US Army 
W911NF1920104, and ONR Grant N000141812861.
Part of this work was done while Ruosong Wang and Lin F. Yang were visiting the Simons Institute for the Theory of Computing.
\bibliographystyle{alpha}
\bibliography{ref}
\end{document}

%% file: abs.tex
\begin{abstract}
Recently there is a surge of interest in understanding the horizon-dependence of the sample complexity in reinforcement learning (RL). Notably, for an RL environment with horizon length $H$, previous work have shown that there is a probably approximately correct (PAC) algorithm that learns an $O(1)$-optimal policy using $\mathrm{polylog}(H)$ episodes of environment interactions when the number of states and actions is fixed. It is yet unknown whether the $\mathrm{polylog}(H)$ dependence is necessary or not. In this work, we resolve this question by developing an algorithm that achieves the same PAC guarantee while using only $O(1)$ episodes of environment interactions, completely settling the horizon-dependence of the sample complexity in RL. We achieve this bound by (i) establishing a connection between value functions in discounted and finite-horizon Markov decision processes (MDPs) and (ii) a novel perturbation analysis in MDPs. We believe our new techniques are of independent interest and could be applied in related questions in RL. 
\end{abstract}

%% file: intro.tex
\section{Introduction}

Reinforcement learning (RL) is one of the most important paradigms in machine learning.
What makes RL different from other paradigms is that it models the long-term effects in decision-making problems. 
For instance, in a finite-horizon Markov decision process (MDP), which is one of the most fundamental models for RL,  an agent interacts with the environment for a total of $H$ steps and receives a sequence of $H$ random reward values, along with stochastic state transitions, as feedback.
The goal of the agent is to find a policy to maximize the \emph{expected sum} of these rewards values instead of any single one of them.
Since decisions made at early stages could significantly impact the future, the agent must take possible future transitions into consideration when choosing the policy. 
On the other hand, when $H=1$, RL reduces to the contextual bandits problem in which it suffices to act myopically to achieve optimality.

Due to the important role of the horizon length in RL, Jiang and Agarwal~\cite{jiang2018open}  propose to study how the sample complexity of RL depends on the horizon length. 
 More formally, let us consider the episodic RL setting, where the horizon length is $H$ and the underlying MDP has unknown and time invariant transition probabilities and rewards. 
Starting from some unknown but fixed  initial state distribution, how many episodes of interaction with the MDP do we need to learn an $\epsilon$-optimal policy, whose \emph{value} (the expected sum of rewards collected by the policy) differs from the optimal policy by at most $\epsilon V_{\max}$?
Here $V_{\max}$ is the maximum possible sum of rewards in an episode and $\epsilon > 0$ is the target accuracy. 
Jiang and Agarwal~\cite{jiang2018open} conjecture that the number of episodes required by the above problem depend polynomially on $H$ even when the total number of state-action pairs is a constant. 
Wang et al.~\cite{wang2020long} refute this conjecture by showing that the correct sample complexity of episodic RL is  at most $\poly((\log H)/\epsilon)$ when there are constant number of states and actions. 
The result in~\cite{wang2020long} shows a surprising fact that one can achieve similar sample efficiency regardless of the horizon length of the RL problem. 
A number of recent results \cite{zhang2020reinforcement, menard2021ucb, zhang2020nearly, ren2021nearly} additionally improve the result in ~\cite{wang2020long} to obtain more (computationally or statistically) efficient algorithm and/or in more general settings. However, these results all have the $\poly((\log H)/\epsilon)$ factor in their sample complexity, seemingly implying that the $\mathrm{polylog}(H)$ dependence is necessary.
Hence, it is natural to ask:
\begin{center}
\emph{What is the exact dependence on the horizon length $H$ of the sample complexity in episodic RL?}
\end{center}
The goal of the current paper is to answer the above question. 

Compared to other problems in machine learning, RL is challenging mainly because it requires an effective exploration mechanism. 
For example, to obtain samples from some state in an MDP, one needs to first learn a policy that can reach that state with good probability. 
In order to decouple the sample complexity of exploration and that of learning a near-optimal policies, 
there is a line of research (see, e.g., \cite{kearns1999finite, azar2013minimax, sidford2018near, wang2017randomized, yang2019sample, li2020breaking}) studying the \emph{generative model} setting, where a learner is able to query as many samples as possible from each state-action pair, circumventing the issue of exploration.
The above sample complexity question is still well-posed even in the generative model.
However, to achieve fair comparison with the RL setting, we measure the sample complexity as the \emph{number of batches}, where a batch of queries corresponds to $H$ queries in the generative model. Indeed, in the episodic RL setting, the learner can also obtain $H$ samples in each episode.
Nevertheless, we are not aware of any result addressing the above question even with a generative model.

\input{result.tex}

\input{rel}

%% file: result.tex
\subsection{Our Results}

In this paper, we settle the horizon-dependence of the sample complexity in episodic RL by developing an algorithm whose sample complexity is completely independent of the planning horizon $H$.
We state our results in the following theorems, which target RL with an $H$-horizon MDP (formally defined in Section~\ref{sec:pre}) with state space $\states$ and action space $\actions$.

\begin{theorem}[Informal version of Corollary~\ref{corollary:rl_main}]\label{thm:rl}
Suppose the reward at each time step is non-negative and the total reward of each episode is upper bounded by 1.\footnote{Without loss of generality, we set $V_{\max}=1$.}
Given a target accuracy $\epsilon > 0$ and a failure probability $\delta > 0$, our algorithm returns an $\epsilon$-optimal \emph{non-stationary} policy with probability at least $1- \delta$ by sampling at most $(|\states||\actions|)^{O(|\states|)} \cdot \log(1/ \delta) / \epsilon^5$ episodes, where $|\states|$ is number of states and $|\actions|$ is the number of actions. 
\end{theorem}

Notably, when the number states $|\states|$, the number of actions $|\actions|$, the desired accuracy $\epsilon$ and the failure probability $\delta$ are all fixed constants, the (episode) sample complexity of our algorithm is also a constant. 
Our result suggests that episodic RL is possible with a sample complexity that is completely independent of the horizon length $H$.

In fact, we can show a much more general result besides finding the optimal policy.
We actually prove that using the same number of episodes, one can construct an oracle that returns an $\epsilon$-approximation of the value of any non-stationary policy $\pi$, given as the following theorem.
\begin{theorem}[Informal version of Theorem~\ref{thm:rl_oracle}]\label{thm:extend}

In the same setting as Theorem~\ref{thm:rl}, with $(|\states||\actions|)^{O(|\states|)} \cdot \log(1/ \delta) / \epsilon^5$ episodes, with probability at least $1 - \delta$, we can construct an oracle $\mathcal{O}$ such that for \emph{every} given non-stationary policy $\pi$, $\mathcal{O}$ returns an $\epsilon$ approximation of the expected total reward of $\pi$.
\end{theorem}

This theorem suggests that even completely learning the MDP (i.e., being able to simultaneously estimate the value of all non-stationary policies) can be done with a  sample complexity that is independent of $H$.

We now switch to the more powerful generative model and show that a better sample complexity can be achieved. 
In the generative model, the agent can query samples from any state-action pair and the initial state distribution of the environment.
Here, the sample complexity is defined as the total number of batches of queries (a batch corresponds to $H$ queries) to the environment to obtain an $\epsilon$-optimal policy. \footnote{It is well-known that any algorithm requires $\Omega(H)$ queries to find a near-optimal policy (see, e.g., \cite{lattimore2012pac}), and thus it is reasonable to define a single batch to be  $H$ queries. }

\begin{theorem}[Informal version of Theorem~\ref{thm:gen_model_main}]\label{thm:gen_model}
Suppose the reward at each time step is non-negative and the total reward of each episode is upper bounded by 1.
Given a target accuracy $\epsilon > 0$ and a failure probability $\delta > 0$, our algorithm returns an $\epsilon$-optimal \emph{non-stationary}  policy with probability at least $1- \delta$ by sampling at most $O(|\states|^6 |\actions|^4 \log(1 / \delta) / \epsilon^3)$ batches of queries in the \emph{generative model}, where $|\states|$ is number of states and $|\actions|$ is the number of actions. 
\end{theorem}

Compared to the result in Theorem~\ref{thm:rl}, the sample complexity in Theorem~\ref{thm:gen_model_main} has polynomial dependence on the number of states and has better dependence on the desired accuracy $\epsilon$. 

We remark that although our algorithms in Theorem~\ref{thm:rl} and Theorem~\ref{thm:gen_model} achieve tight dependence in terms of $H$, it does not aim to  tighten the dependence on the number of states $|\states|$, the number of actions $|\actions|$ and the desired accuracy $\epsilon$.
In fact, the sample complexity of our algorithms have much worse dependence in terms of $|\states|$, $|\actions|$ and $\epsilon$ compared to prior work. 
See Section~\ref{sec:related} for a detailed discussion on prior work, and Section~\ref{sec:tech} for an overview of our techniques.

%% file: rel.tex
\subsection{Related Work}\label{sec:related}
In this section, we discuss related work on the sample complexity of RL in the tabular setting, where the number of states $\abs{\states}$, the number of actions $\abs{\actions}$, and the horizon length $H$ are all assumed to be finite. There is a long line of research studying the sample complexity in tabular RL~\cite{kearns2002near,brafman2002r,kakade2003sample,strehl2006pac,strehl2008analysis,kolter2009near,bartlett2009regal,jaksch2010near,szita2010model,lattimore2012pac,osband2013more,dann2015sample,azar2017minimax,dann2017unifying,osband2017posterior,jin2018q,fruit2018near,talebi2018variance,dann2019policy,dong2019q,simchowitz2019non,russo2019worst,zhang2019regret,zhang2020almost,yang2021q,pacchiano2020optimism,neu2020unifying, wang2020long, zhang2020reinforcement, menard2021ucb, zhang2020nearly, ren2021nearly}.
In most prior work, up to a scaling factor, the standard assumption on the reward values is
that $r_h \in [0,1/H]$ for all $h$ and hence $\sum_{h} r_h \in [0, 1]$, where $r_h$ is the reward value collected at step $h$ of an episode.
We refer this assumption as the \emph{reward uniformity} assumption.
However, Jiang and Agarwal~\cite{jiang2018open} point out that one should only impose an upper bound on the summation of the reward values, i.e., $\sum_{h\in[H]} r_h\le 1$, which we refer as the \emph{bounded total reward} assumption, instead of imposing uniformity which could be much stronger.
This new assumption allows a fair comparison between long-horizon and short-horizon problems, and is also more natural when the reward signal is sparse.
Also note that algorithms designed under the reward uniformity assumption can be modified to work under the bounded total reward assumption. 
However, the sample complexity is usually worsen by a $\poly(H)$ factor.
Many existing results in tabular RL focuses on regret
minimization where the goal is to collect maximum cumulative rewards for
a limit number of interactions with the environment.
To draw comparison with our results, we convert them,  using standard techniques,  to the PAC setting, where the algorithm aims to obtain an $\epsilon$-optimal policy while minimizing the number of episodes of environment interactions (or batches of queries if in the generative model).


Under the reward uniformity assumption, a line of work have attempted to provide tight sample complexity bounds~\cite{azar2017minimax,dann2015sample,dann2017unifying,dann2019policy,jin2018q,osband2017posterior}.
To obtain an $\varepsilon$-optimal policy, state-of-the-art results show
that $\widetilde{O}\left(\frac{\abs{\states}\abs{\actions}}{\varepsilon^2} +
  \frac{\poly\left(\abs{\states},\abs{\actions},H\right)}{\varepsilon}\right)$\footnote{$\widetilde{O}\left(\cdot\right)$ omits logarithmic factors.}
episodes suffice~\cite{dann2019policy,azar2017minimax}.
In particular, the first term matches the lower bound $\Omega\left(\frac{\abs{\states}\abs{\actions}}{\varepsilon^2}\right)$ up to logarithmic factors~\cite{dann2015sample,osband2016on,azar2017minimax}
while the second term has least linear dependence on $H$.
Moreover, converting these bounds to the bounded total reward setting induces extra $\poly(H)$ factors.
For instance, the sample complexity in~\cite{azar2017minimax,dann2019policy} will become $\widetilde{O}\left(\frac{\abs{\states}\abs{\actions}H^2}{\varepsilon^2} +
  \frac{\poly\left(\abs{\states},\abs{\actions},H\right)}{\varepsilon}\right)$ under the bound total reward assumption

Recently, there is a surge of interest of designing algorithms under the bounded total reward assumption.
In particular, Zanette and  Brunskill~\cite{zanette2019tighter} develop an algorithm which enjoys
a sample complexity of $\widetilde{O}\left(\frac{\abs{\states}\abs{\actions}}{\varepsilon^2} + \frac{\poly\left(\abs{\states},\abs{\actions},H\right)}{\varepsilon}\right)$
, where the second term still scales \emph{polynomially} with $H$.
Wang et al.~\cite{wang2020long}  show that it is possible to obtain an $\epsilon$-optimal policy with a sample complexity of $\poly(|\states||\actions|/\epsilon)\cdot \log^3 H$, establishing the first sample complexity with $\mathrm{polylog}(H)$ dependence on the horizon length $H$.
They achieve this result by using the following ideas: (1) samples collected by different policies can be reused to evaluate other policies; (2) to evaluate all policies in a finite set $\Pi$, the number of sample episodes required is at most $\poly(|\states||\actions|/\epsilon)\cdot \log|\Pi| \cdot \log^2H$; (3) establish a set of policies $\Pi$ that contains at least one $\epsilon$-optimal policy for any MDP by using an $\epsilon$-nets over the reward values and the transition probabilities.
Here $\Pi$ contains all optimal non-stationary policies of each MDP in the $\epsilon$-net.
The accuracy of the $\epsilon$-net needs to be at least $\poly(\epsilon/H)$ and hence $|\Pi| = \poly(\epsilon/H)^{|\states||\actions|}$, which induces an inevitable $\log H$ factor. 
Other $\log H$ factors come from Step (2) where they use a potential function which increases from $0$ to $H$ to measure the progress of the algorithm.
It is unclear how to remove any of the log factors in their sample complexity upper bound, and they also conjecture that the optimal sample complexity could be $\Omega(\mathrm{polylog}(H))$ which is refuted by our new result.

Another recent work~\cite{zhang2020reinforcement} obtains a more efficient algorithm that achieves a sample complexity of $\widetilde{O}\left(
\left(\frac{|\states||\actions|}{\epsilon^2} +\frac{|\states|^2|\actions|}{\epsilon}  \right)\cdot \poly\log H\right)$. 
Zhang et al.~\cite{zhang2020reinforcement} achieve such a sample complexity by using a novel upper confidence bound (UCB) analysis on a model-based algorithm. In particular, they use a recursive structure to bound the total variance in the MDP, which in turn bounds the final sample complexity. However, their bound critically relies on the fact that the total variance of value function moments is upper bounded by $H$ which inevitably induces $\mathrm{polylog}(H)$ factors. It is unlikely to get rid of all  $\mathrm{polylog}(H)$ factors in their sample complexity by following their approaches.


In the generative model, 
there is also a long line of research studying the sample complexity of finding near-optimal policies. 
See, e.g., \cite{kearns1999finite, kakade2003sample, singh1994upper, azar2013minimax, sidford2018variance, sidford2018near,agarwal2019optimality,li2020breaking}. 
However, most of these work are targeting on the infinite-horizon discounted setting, in which case the problem becomes much easier.
This is because in the infinite-horizon discounted setting, there always exists a stationary optimal policy, and the total number of such policies depends only on the number of states $|\states|$ and the number of actions $|\actions|$. 
Sidford et al.~\cite{sidford2018near} develop an algorithm based on a variance-reduced estimator which uses $\widetilde{O}\left(\frac{|\states||\actions|}{\epsilon^2} + H^2|\states||\actions|\right)$ batches of queries to find an $\epsilon$-optimal policy under the reward uniformity assumption. 
However, their result relies on splitting samples into $H$ sub-groups, and therefore, the lower order term of their sample complexity fundamentally depends on $H$.
To the best of our knowledge, even in the stronger generative model, no algorithm for finite-horizon MDPs  achieves a sample complexity that is independent of $H$ under the bounded total reward assumption.

%% file: prelim.tex
\section{Preliminaries}
\label{sec:pre}
\paragraph{Notations.}
Throughout this paper, for a given positive integer $H$, we use $[H]$ to denote the set $\{0, 1, \ldots, H - 1\}$.
For a condition $\mathcal{E}$, we use $\indict[\mathcal{E}]$ to denote the indicator function, i.e., $\indict[\mathcal{E}] = 1$ if $\mathcal{E}$ holds and $\indict[\mathcal{E}] = 0$ otherwise.

For a random variable $X$ and a real number $\varepsilon \in [0, 1]$, its $\varepsilon$-quantile $\qtile_\varepsilon(X)$ is defined so that
\[
\qtile_\varepsilon(X) = \sup \{x \mid \Pr[X \ge x] \ge \varepsilon \}.
\]

\paragraph{Markov Chains.}
Let $\mc = (\states, \trans, \mu)$ be a Markov chain where $\states$ is the state space, $\trans : \states \to \Delta(\states)$ is the transition operator and $\mu \in \Delta(\states)$ is the initial state distribution. A Markov chain $\mc$ induces a sequence of random states
\[
s_0, s_1, \ldots
\]
where for each $s_0 \sim \mu$ and $s_{h + 1} \sim P(s_h)$ for each $h \ge 0$. 

\paragraph{Finite-Horizon Markov Decision Processes.}
Let $\mdp =\left(\states, \actions, \trans ,R, H, \mu\right)$ be an $H$-horizon \emph{Markov Decision Process} (MDP)
where $\states$ is the finite state space, 
$\actions$ is the finite action space, 
$\trans: \states \times \actions \rightarrow \Delta \left(\states\right)$ is the transition operator which takes a state-action pair and returns a distribution over states, 
$R : \states \times \actions \rightarrow \Delta\left( \mathbb{R} \right)$ is the reward distribution,
$H \in \mathbb{Z}_+$ is the planning horizon  (episode length),
and $\mu \in \Delta\left(\states\right)$ is the initial state distribution. 
In the learning setting, $P$, $R$ and $\mu$, are unknown but fixed, and $\states$, $\actions$ and $H$ are known to the learner.
Throughout, we assume $H=\Omega(|\states|\log|\states|)$ since otherwise we can apply existing algorithms (e.g.~\cite{zhang2020reinforcement}) to obtain a sample complexity that is independent of $H$.
Throughout the paper, for a state $s \in \states$, we occasionally abuse notation and use $s$ to denote the deterministic distribution that always takes $s$. 
\paragraph{Interacting with the MDP.} We now introduce how a RL agent (or an algorithm) interacts with an unknown MDP.
 In the setting of Theorem~\ref{thm:rl}, in each episode, the agent first receives an initial state $s_0 \sim \mu$.
For each time step $h \in [H]$, the agent first decides an action $a_h \in \actions$, and then observes and moves to $s_{h + 1} \sim P(s_h, a_h)$ and obtains a reward $r_{h} \sim R(s_h, a_h)$. 
The episode stops at time $H$, where the final state is $s_H$.
Note that even if the agent decides to stop before time $H$, e.g. at time $\sqrt{H}$, this still counts as one full episode. 

In the generative model setting (as in Theorem~\ref{thm:gen_model}), the agent is allowed to start from any $s_0 \in \mathcal{S}$ instead of $s_0 \sim \mu$, and is allowed to restart at any time. 
The sample complexity is defined as the number of batches of queries (each batch consists of $H$ queries) the agent uses.

\paragraph{Policy Class.} 
The final goal of RL is to output a good policy $\pi$ with respect to the unknown MDP that the agent interacts with. 
In this paper, we consider (non-stationary) deterministic policies $\pi$ which choose an action $a$ based on the current state $s \in \states$ and the time step $h \in [H]$. 
Formally, $\pi = \{\pi_h\}_{h = 0}^{H - 1}$ where for each $h \in [H]$, $\pi_h : \states \to \actions$ maps a given state to an action.
The policy $\pi$ induces a (random) trajectory \[(s_0,a_0,r_0),(s_1,a_1,r_1),\ldots,(s_{H-1},a_{H-1},r_{H-1}), s_H,\]
where $s_0 \sim \mu$, $a_0 = \pi_0(s_0)$, $r_0 \sim R(s_0,a_0)$, $s_1 \sim \trans(s_0,a_0)$, $a_1 = \pi_1(s_1)$, $r_1 \sim R(s_1, a_1)$, $\ldots$, $a_{H-1} = \pi_{H-1}(s_{H-1})$, $r_{H-1} \sim R(s_{H-1}, a_{H-1})$ and $s_H\sim \trans(s_{H-1},a_{H-1})$. 
Throughout the paper, all policies are assumed to be deterministic unless specified otherwise.

\paragraph{Value Functions.}To measure the performance of a policy $\pi$, we define the \emph{value} of a policy as \[V^{\pi}_{\mdp, H} = \expect \left[ \sum_{h = 0}^{H-1} r_h \right].\]
Note that in general a policy does not need to be deterministic, and it can even depend on the history in which case a policy chooses an action at time $h$ based on the entire transition history before $h$.
It is well-known (see e.g. \cite{puterman1994markov}) that the optimal value of $\mdp$ can be achieved by a non-stationary deterministic optimal policy $\pi^*$. 
Hence, we only need to consider non-stationary deterministic policies.

The goal of the agent is then to find a policy $\pi$ with $V^{\pi}_{M, H} \ge V^{\pi^*}_{M, H} - \epsilon$, i.e., obtain an $\epsilon$-optimal policy, while minimizing the number of episodes of environment interactions (or batches of queries if in the generative model). 

For a policy $\pi$, we also define 
\[
\qtile^{\pi}_{\delta}(s,a) = \qtile_{\delta}\left[\sum_{t=0}^{H}\indict[(s,a) = (s_t, a_t)]\right].
\]
to be the $\delta$-quantile of the visitation frequency of a state-action pair $(s,a)$.

\paragraph{Stationary Policies.}
For the sake of the analysis, we shall also consider stationary policies.
A stationary deterministic policy $\pi$ chooses an action $a$ solely based on the current state $s \in \states$, i.e, $\pi_0 = \pi_1= \ldots = \pi_{H-1}$.
We use $\Pi_{\sta}$ to denote the set of all stationary policies. Note that $\left | \Pi_{\sta} \right| =  |\actions|^{|\states|}$.


We remark that when the horizon length $H$ is finite, the value of the best stationary policy and that of the best non-stationary policy can differ significantly.
Consider the case when there are two states $s, s'$ and two actions $a, a'$. Starting from $s$, taking action $a$ will go back to $s$ and obtain a reward of $1/H$, while taking action $a'$ will go to $s'$ and obtain a reward of $1$. Taking any action at $s'$ will go back to $s'$ with reward 0. In this case, the optimal non-stationary policy has value $2 - 1/H$ since it can exit $s$ at time $H$, while deterministic stationary policies can only have value $\leq 1$ (even random stationary policy can only achieve value $2 - \Omega(1)$).

For a stationary policy $\pi : \states \to \actions$, we use $\mdp^{\pi} = (\states, \trans^{\pi}, \mu)$ to denote the Markov chain induced by $\mdp$ and $\pi$, where the transition operator $\trans^{\pi}$ is defined so that
\[
\trans^{\pi}(s' \mid s) = \trans(s' \mid s, \pi(s)).
\]

\paragraph{Assumption on Rewards.} Below, we formally introduce the \emph{bounded total reward assumption}.
\begin{assumption}[Bounded Total Reward]\label{assump:reward}
For any policy $\pi$, and a random trajectory $((s_0, a_0, r_0),$ $(s_1, a_1, r_1), \ldots, (s_{H-1}, a_{H-1}, r_{H-1}), s_H)$ induced by $\pi$, we have
\begin{itemize}
\item  $r_h \in [0, 1]$ for all $h \in [H]$; and
\item $\sum_{h = 0}^{H - 1} r_h \le 1$ almost surely.
\end{itemize}
\end{assumption}
As discussed in the previous section, this assumption is more general than the standard reward uniformity assumption, where $r_h\in [0,1/H]$, for all $h\in [H]$. Throughout this paper, we assume that the above assumption holds for the unknown MDP that the agent interacts with. 

The above assumption in fact implies a very interesting consequence on the reward values.
\begin{lemma}\label{lem:reward}
Under Assumption~\ref{assump:reward}, for any $\mdp =\left(\states, \actions, \trans ,R, H, \mu\right)$ with $H \ge |\states|$, for any $(s, a) \in \states \times \actions$, if there exists a (possibly non-stationary) policy $\pi$ such that for the random trajectory
\[(s_0,a_0,r_0),(s_1,a_1,r_1),\ldots,(s_{H-1},a_{H-1},r_{H-1}), s_H\]
induced by executing $\pi$ in $\mdp$, we have
\[
\Pr \left[\sum_{h = 0}^{H - 1} \indict[(s_h, a_h) = (s, a)] > 1\right] \ge \epsilon
\]
for some $\epsilon > 0$,
then $R(s, a) \le 2|\states| / H$ almost surely and therefore $\expect[(R(s, a))^2] \le 4|\states|^2 / H^2$. 
\end{lemma}
\begin{proof}
By the assumption, there exists a trajectory
\[
((s_0, a_0), (s_1, a_1), \ldots, (s_{H - 1}, a_{H - 1}), s_H)
\]
such there exists $0 \le h_1 < h_2 < H$ with $(s_{h_1}, a_{h_1}) = (s, a)$ and $(s_{h_2}, a_{h_2}) = (s, a)$.
Moreover, 
\[
\mu(s_0) \prod_{h = 0}^{h_2} P(s_{h + 1} \mid s_h, a_h) > 0.
\]
We may assume $h_1 < |\states|$ and $h_2 - h_1 \le |\states|$, since otherwise we can replace sub-trajectories that start and end with the same state by that state, and the resulting trajectory can still appear with strictly positive probability.
Now consider the policy $\widehat{\pi}$ which is defined so that for each $h < h_1$, $\widehat{\pi}_h(s_h) = a_h$ and for each $t \in [h_2 - h_1]$,
\[
\widehat{\pi}_{h_1 + t}(s_{h_1 + t}) = \widehat{\pi}_{h_1 + (h_2 - h_1) + t}(s_{h_1 + t}) = \widehat{\pi}_{h_1 + 2(h_2 - h_1) + t}(s_{h_1 + t})  = \cdots = a_{h_1 + t},
\]
i.e., repeating the trajectory's actions in $[h_1, h_2]$ indefinitely.
$\widehat{\pi}$ is defined arbitrarily for other states and time steps. 

By executing $\widehat{\pi}$, with strictly positive probability, $(s, a)$ is visited for $\floor{H / |\states|} \ge H / (2|\states|)$ times.
Therefore, by Assumption~\ref{assump:reward}, $R(s, a) \le 2|\states| / H$ with probability $1$ and thus 
$\expect[(R(s, a))^2] \le 4|\states|^2 / H^2$. 

\end{proof}

\paragraph{Discounted Markov Decision Processes.}
We also introduce another variant of MDP, discounted MDP, which is specified by $\mdp =\left(\states, \actions, \trans ,R, \gamma, \mu\right)$, where $\gamma\in(0,1)$ is a discount factor and all other components have the same meaning as in an $H$-horizon MDP.
Note that we consider discounted MDP only for the purpose of analysis, and the unknown MDP that the agent interacts with is always a finite-horizon MDP. 
The difference between a discounted MDP and an $H$-horizon MDP is that discounted MDPs have an infinite horizon length, i.e., the length of a trajectory can be infinite. 
Hence, to measure the value of a discounted MDP, suppose policy $\pi$ obtains a random trajectory, 
\[
(s_0,a_0,r_0),(s_1,a_1,r_1),\ldots,(s_{h},a_{h},r_{h}),\ldots,
\]
we denote 
\[V^{\pi}_{\mdp, \gamma} = \expect \left[ \sum_{h = 0}^{\infty} \gamma^h r_h \mid s=s_0\right]\]
 as the discounted value of $\pi$.
Throughout, for a (discounted or finite-horizon)  MDP $\mdp =\left(\states, \actions, \trans ,R, \cdot, \mu\right)$, we simply denote $V^{\pi}_{\mdp, H}$ as the value function of $\left(\states, \actions, \trans ,R, H, \mu\right)$
and $V^{\pi}_{\mdp, \gamma}$ as the value function of $\left(\states, \actions, \trans ,R, \gamma, \mu\right)$.



%% file: tech.tex
\section{Technical Overview}\label{sec:tech}
In this section, we discuss the techniques for establishing our results.
To introduce the high-level ideas, we first start with the simpler setting, the generative model, where exploration is not a concern.
We then switch to the more challenging RL setting, where we need to carefully design policies to explore the state-action space so that a good policy can be learned.
For simplicity, throughout the discussion in this section, we assume $|\states|$, $|\actions|$ and $1 / \epsilon$ are all constants. 

\paragraph{Algorithm and Analysis in the Generative Model.}
Our algorithm in the generative model is conceptually simple: for each state-action pair $(s, a)$, we draw $O(H)$ samples from $P(s, a)$ and $R(s, a)$ and then return the optimal policy with respect to the empirical model $\widehat{M}$ which is obtained by using the empirical estimators for $P$ and $R$ (denoted as $\widehat{P}$ and $\widehat{R}$). 
Here for simplicity, we assume $R = \widehat{R}$ which allows us to focus on the estimation error induced by the transition probabilities.
Moreover, we assume that $P$ differs from $\widehat{P}$ only for a single state-action pair $(s, a)$.  
To further simplify the discussion, we assume that there are only two different states on the support of $P(s' \mid s, a)$ (say $s_1$ and $s_2$). 

In order to prove the correctness of the algorithm, we show that for any policy $\pi$, the value of $\pi$ in the empirical model $\widehat{M}$ is close to that in the true model $M$.
However, standard analysis based on dynamic progarmming shows that the difference between the value of $\pi$ in $\widehat{M}$ and that in $M$ could be as large as $H$ times the estimation error on $P(s, a)$, which is clearly insufficient for obtaining an algorithm which uses $O(1)$ batches of queries. 
Our main idea here is to show that for most trajectories $T$, the probability of $T$ in the empirical model $\widehat{M}$ is a {\em multiplicative approximation} to that in the true model $M$ with constant approximation ratio. 

To establish the multiplicative approximation guarantee, our observation is that one should consider $s_1$ and $s_2$, the two states on the support of $P(s, a)$, as a whole.  
To see this, consider the case where $P(s_1 \mid s, a) = P(s_2 \mid s, a) = 1/2$. Again, the additive estimation errors on both $P(s_1 \mid s, a)$ and $P(s_2 \mid s, a)$ are roughly $O(1 / \sqrt{H})$. 
Now, consider a trajectory that visits both $(s, a, s_1)$ and$(s, a, s_2)$ for $H / 2$ times. 
Note that the multiplicative approximation ratio between $\widehat{P}(s' \mid s, a)^{H / 2}$ and $P(s' \mid s, a)^{H / 2}$ could be as large as $\exp(\sqrt{H})$, for both $s' = s_1$ and $s' = s_2$. 
However, since  $\widehat{P}(s_1 \mid s, a) + \widehat{P}(s_2 \mid s, a) = 1$ as the empirical estimator $\widehat{P}(s, a)$ is still a probability distribution, it must be the case that $\widehat{P}(s_1 \mid s, a)/P(s_1 \mid s, a) = 1 - 2\delta$ and $\widehat{P}(s_2 \mid s, a)/{P}(s_2 \mid s, a) = 1 + 2\delta$ where $\delta = P(s_1 \mid s, a) - \widehat{P}(s_1 \mid s, a)$ and thus $|\delta| \le O(1 / \sqrt{H})$. 
Since $(1+2\delta)^{H/2}(1-2\delta)^{H/2}=(1 - 4 \delta^2)^{H/2}$ is a constant, 
$(\widehat{P}(s_1 \mid s, a))^{H / 2} (\widehat{P}(s_2 \mid s, a))^{H / 2}$ is a constant factor approximation to the true probability $(P(s_1 \mid s, a))^{H / 2} (P(s_2 \mid s, a))^{H / 2}$ due to cancellation. 

In our analysis, to formalize the above intuition, for each trajectory $T$, we take $T$ into consideration only when $|m_T(s, a, s') - P(s' \mid s, a) \cdot m_T(s, a) |\le O(\sqrt{P(s' \mid s, a) \cdot H })$
for both $s'=s_1$ and $s'=s_2$. Here $m_T(s, a)$ is the number of times that $(s, a)$ is visited on $T$ and $m_T(s, a, s')$ is the number of times that $(s, a, s')$ is visited on $T$.
By Chebyshev's inequality (see Lemma~\ref{lemma:martingal_concent} for details), we only ignore a small subset of trajectories whose total probability can be upper bounded by a constant. 
For the remaining trajectories, it can be shown that
$
\widehat{P}(s_1 \mid s, a)^{m_T(s, a, s_1)}  \cdot \widehat{P}(s_2 \mid s, a)^{m_T(s, a, s_2)} 
$
is a constant factor approximation to 
$
P(s_1 \mid s, a)^{m_T(s, a, s_1)}   \cdot P(s_2 \mid s, a)^{m_T(s, a, s_2)} $ so long as $|\widehat{P}(s, a, s') - P(s, a, s')| \le O(\sqrt{P(s, a, s') / H})$ for both $s'  = s_1$ and $s' = s_2$ due to the cancellation mentioned above. The formal argument is given in Lemma~\ref{lem:mult-add}.
Note that using $O(H)$ samples, $|\widehat{P}(s, a, s') - P(s, a, s')| \le O(\sqrt{P(s, a, s') / H})$ holds only when $P(s, a, s') \ge \Omega(1 / H)$.  
On the other hand, we can also ignore trajectories that visit $(s, a, s')$ with $P(s, a, s') \le O(1 / H)$ since such trajectories have negligible cumulative probability by Markov's inequality (formalized in Lemma~\ref{lem:upperbound_on_visits}). 

The above analysis can be readily generalized to handle perturbation on the transition probabilities of multiple state-action pairs, and to handle the case when the transition operator $P(\cdot \mid s,a)$ is not supported on two states. 
In summary, by using $O(H)$ samples for each state-action pair $(s, a)$, the empirical model $\widehat{M}$ provides a constant factor approximation to the probabilities of all trajectories, except for a small subset of them whose cumulative probability can be upper bounded by a constant.
Hence, for all policies, the empirical model provides an accurate estimate to its value and thus, the optimal policy with respect to the empirical model is near-optimal. 

\paragraph{Exploration by Stationary Policies.}
In the discussion above, we heavily rely on the ability of the generative model to obtain $\Omega(H)$ samples for each state-action pair. However, for the RL setting, it is not possible to reach every state-action pair freely. 
Although each trajectory contains $H$ state-action-state tuples (corresponding to a batch of queries in the generative model), these samples may not cover states that are crucial for learning an optimal policy.
Indeed, one could use all possible deterministic non-stationary policies to collect samples, which shall then cover the whole state-action space.
Unfortunately, such a na\"ive method introduces a dependence on the number of non-stationary policies which is exponential in $H$. 
The sample complexity of other existing methods in the literature also inevitably depends on $H$ as their sample complexity intrinsically depends on the number of non-stationary policies.

In this work, 
we adopt a completely different approach for exploration.
Our new idea is to show that if there exists a non-stationary policy that visits $(s, a)$ for $f$ times in expectation, then there exists a stationary policy that visit $(s, a)$ for $f / \exp\left( O(|\states|\log|\states|) \right)$ times in expectation.
This is formalized in Lemma~\ref{cor:non_stationary_nearly_optimal}.
If the above claim is true, then intuitively, one can simply enumerate all stationary policies and sample $\exp\left( O(|\states|\log|\states|)\right)$ trajectories using each of them to obtain $f$ samples of $(s, a)$. 
Note that there are only $|\actions|^{|\states|}$ stationary policies, which is completely independent of $H$.
In order to prove the above claim, we show that for any stationary policy $\pi$, its value in the infinite-horizon discounted setting is close to that in the finite-horizon undiscounted setting (up to a factor of $\exp\left( O(|\states|\log|\states|) \right)$) by using a properly chosen discount factor (see Lemma~\ref{lem:discount_finite_compare}). 
Note that this implies the correctness of the above claim since there always exists a stationary optimal policy in the infinite-horizon discounted setting. 

In order to show the value of a stationary policy in the infinite-horizon discounted setting is close to that in the finite-horizon setting, we study reaching probabilities in time invariant Markov chains. 
In particular, we show in Lemma~\ref{lem:mc_main} that in a time invariant Markov chain, for any $H \ge |\states|$, the probability of reaching a specific state $s$ within $H$ steps is close the probability of reaching $s$ within $4H$ steps, up to a factor of $\exp\left( O(|\states|\log|\states|) \right)$.
Previous literature on time invariant Markov chains mostly focus on the asymptotic behavior, and as far as we are aware, we are the first to prove the above claim. 
Note that this claim directly establishes a connection between the value of a stationary policy in the infinite-horizon discounted setting and that in the finite-horizon setting.
Moreover, as a direct consequence of the above claim, we can show that if $H > 2|\states|$, the value of a stationary policy within $H$ steps is close to that of the same policy within $H / 2$ steps, up to a factor of $\exp\left( O(|\states|\log|\states|) \right)$. This consequence is crucial for later parts of the analysis.

\paragraph{From Expectation to Quantile.}
The above analysis shows that if there exists a non-stationary policy that visits $(s, a)$ for $f$ times in expectation, then our algorithm, which uses all stationary policies to collect samples, visits $(s, a)$ for $f / \exp\left( O(|\states|\log|\states|) \right)$ times in expectation. However, this does not necessarily mean that one can obtain $f$ samples of $(s, a)$ by sampling $\exp\left( O(|\states|\log|\states|) \right)$ trajectories using our algorithm with good probability.
To see this, consider the case when our policy visits $(s, a)$ for $H$ times with probability $1 / \sqrt{H}$ and does not visit $(s, a)$ with probability $1  - 1 / \sqrt{H}$. 
In this case, our policy may not obtain even a single sample for $(s, a)$ unless one rollouts the policy for $O(\sqrt{H})$ times. 
Therefore, instead of obtaining a visitation frequency guarantee which holds in expectation, it is more desirable to obtain a visitation frequency guarantee that holds with good probability.

To resolve this issue, we establish a connection between the expectation and the $\epsilon$-quantile of the visitation frequency of a stationary policy. 
We note that such a connection could not hold without any restriction.
To see this, consider a policy that visits $(s, a)$ for $H$ times with probability $\epsilon / 2$. 
In this case, the expected visitation frequency is $\epsilon H / 2$ while the $\epsilon$-quantile is zero. 
On the other hand, suppose the initial state $s_0 = s$ almost surely, then such a connection is easy to establish by using the Martingale Stopping Theorem. 
In particular, in Corollary~\ref{cor:quantile_comparison}, we show that if there exists a non-stationary policy that visits $(s, a)$ for $f$ times with probability $\epsilon$, then there exists a stationary policy that visits $(s, a)$ for $\epsilon f / \exp\left( O(|\states|\log|\states|) \right)$ times with constant probability, when  the initial state $s_0 = s$ almost surely.

In general, the initial state $s_0$ comes from a distribution $\mu$ and could be different from $s$ with high probability. 
To tackle this issue, in our algorithm, we simultaneously enumerate two stationary policies $\pi_1$ and $\pi_2$.
$\pi_1$ should be thought as the policy that visits $(s, a)$ with highest probability within $H / 2$ steps starting from the initial state distribution $\mu$, and $\pi_2$ should be thought as the policy that maximizes the $\epsilon$-quantile of the visitation frequency of $(s, a)$ within $H / 2$ steps when $s_0 = s$.
In our algorithm, we execute $\pi_1$ before $(s, a)$ is visited for the first time, and switch to $\pi_2$ once $(s, a)$ has been visited. 
Intuitively, we first use $\pi_1$ to reach $(s, a)$ for the first time and then use $\pi_2$ to collect as many samples as possible. 
As mentioned above, the value of a stationary policy within $H$ steps is close to the value of the same policy within $H / 2$ steps, up to a factor of $\exp\left( O(|\states|\log|\states|) \right)$.
Thus, by sampling the above policy (formed by concatenating $\pi_1$ and $\pi_2$) for $\exp\left( O(|\states|\log|\states|) \right) / \epsilon^2$ times, we obtain at least $f$ samples for $(s, a)$, if there exists a non-stationary policy that visits $(s, a)$ for $f$ times with probability $\epsilon$.
This is formalized in Lemma~\ref{lem:stationary_quantile}.

\paragraph{Perturbation Analysis in the RL Setting.}

By the above analysis, suppose $m(s, a)$ is the largest integer such that there exists a non-stationary policy that visits $(s, a)$ with probability $\epsilon$ for $m(s, a)$ times, then our dataset contains $\Omega(m(s, a))$ samples of $(s, a)$. However, $m(s, a)$ could be significantly smaller than $H$ and therefore the perturbation analysis established in the generative model no longer applies here. 
For example, previously we show that if 
$|m_T(s, a, s') - P(s' \mid s, a) \cdot m_T(s, a)| \le O(\sqrt{P(s' \mid s, a) \cdot H })$, then $
\widehat{P}(s_1 \mid s, a)^{m_T(s, a, s_1)}  \cdot \widehat{P}(s_2 \mid s, a)^{m_T(s, a, s_2)} 
$
is a constant factor approximation to 
$
P(s_1 \mid s, a)^{m_T(s, a, s_1)}   \cdot P(s_2 \mid s, a)^{m_T(s, a, s_2)} $ when $|\widehat{P}(s, a, s') - P(s, a, s')| \le O(\sqrt{P(s, a, s') / H})$ for both $s'  = s_1$ and $s' = s_2$. 
However, if $m(s, a) \ll H$, it is hopeless to obtain an estimate $\widehat{P}(s, a, s')$ with $|\widehat{P}(s, a, s') - P(s, a, s')| \le O(\sqrt{P(s, a, s') / H})$.
Fortunately, our perturbation analysis still goes through so long as $m_T(s, a, s') \le P(s' \mid s, a) \cdot m_T(s, a) + O(\sqrt{P(s' \mid s, a) \cdot m(s, a) })$ and $|\widehat{P}(s, a, s') - P(s, a, s')| \le O(\sqrt{P(s, a, s') / m(s, a)})$, i.e., replacing all $H$ appearances with $m(s, a)$. 

The above analysis introduces a final subtlety in our algorithm.
In particular, $m(s, a)$ in the empirical model $\widehat{M}$ could be significantly larger than that in the true model. 
On the other hand, the number of samples of $(s, a)$ in our dataset is at most $O(m(s, a))$ where $m(s, a)$ is defined by the true model. 
This means the value estimated in the empirical model $\widehat{M}$ could be significantly larger than that in the true model $M$. 
 To resolve this issue, we employ the principle of ``pessimism in the face of uncertainty'' and for each policy $\pi$, the estimated value of $\pi$ is set to be the lowest value among all models that lie the confidence set. 
Since the true model always lies in the confidence set, the estimated value is then guaranteed to be close to the true value.

%% file: discount.tex
\section{Properties of Stationary Policies}
In this section, we prove several properties of stationary policies.
In Section~\ref{sec:mc}, we first prove properties regarding the reaching probabilities in Markov chains, and then use them to prove properties for stationary policies in Section~\ref{sec:stationary}.
\subsection{Reaching Probabilities in Markov Chains}\label{sec:mc}
Let $\mc = (\states, \trans, \mu)$ be a Markov chain.
For a positive integer $L$ and a sequence of states $T = (s_0, s_1, \ldots, s_{L - 1}) \in \states^{L}$, we write  
\[
p(T, \mc) =  \mu(s_0) \cdot \prod_{h = 0}^{L- 2}  \trans(s_{h + 1} \mid s_{h})
\]
to denote the probability of $T$ in $C$.
For a state $s \in \states$ and an integer $L \ge 0$, we also write
\[
p_L(s, \mc) = \sum_{(s_0, s_1, \ldots, s_{L-1}) \in \states^{L}} p((s_0, s_1, \ldots, s_{L - 1}, s), C)
\]
to denote the probability of reaching $s$ with exactly $L$ steps. 

Our first lemma shows that for any Markov chain $\mc$, 
for any sequence of $L$ states $T$ with $L > |\states|$, there exists a sequence of $L'$ states $T'$ with $L' \le |\states|$ so that $p(T, \mc) \le p(T', \mc)$.
\begin{lemma}\label{lem:reduction}
Let $\mc = (\states, \trans, \mu)$ be a Markov chain.
For a sequence of $L$ states \[T = (s_0, s_1, \ldots, s_{L - 1}) \in \states^{L}\] with $L > |\states|$,
there exists a sequence of $L'$ states  \[T' = (s_0', s_1', \ldots, s_{L' - 1}') \in \states^{L'}\] 
with $s_{L' - 1}' = s_{L-1}$, $L' \le |\states|$ and $p(T, \mc) \le p(T', \mc)$.
\end{lemma}
\begin{proof}
By pigeonhole principle, since $L >  |\states|$, there exists $0 \le i < j < L$ such that $s_i = s_j$.
Consider the sequence induced by removing $s_i, s_{i + 1}, s_{i + 2}, \ldots, s_{j - 1}$ from $T$, i.e., 
\[
T' = (s_0, s_1, \ldots, s_{i - 1}, s_{j}, s_{j + 1}, \ldots, s_{L - 1}).
\] 
Since $s_i = s_j$, we have 
\[
p(T, \mc)  = \mu(s_0) \cdot  \prod_{h = 0}^{L - 2}  P(s_{h + 1} \mid s_{h}).
\]
and
\[
 p(T', \mc)  \\
= \mu(s_0) \cdot \prod_{h = 0}^{i - 1}  P(s_{h + 1} \mid s_{h})  \cdot    \prod_{h = j}^{L - 2}   P(s_{h + 1} \mid s_{h}).
\]
Therefore, we have $p(T, \mc) \le p(T', \mc)$.
We continue this process until the length is at most $|\states|$.
\end{proof}

Combining Lemma~\ref{lem:reduction} with a simple counting argument directly implies the following lemma, which shows that $\sum_{h = 0}^{4|\states|-1} p_h(s, \mc) \le \exp\left(O(|\states| \log |\states|)\right) \cdot \sum_{h= 0}^{|\states|-1}p_h(s, \mc)$. 
\begin{lemma}\label{lem:sum_unexpanded}
Let $\mc = (\states, \trans, \mu)$ be a Markov chain.
For any $s \in \states$,
\[
\sum_{h = 0}^{4|\states|-1} p_h(s, \mc) \le 4 \cdot |\states|^{4|\states|} \cdot \sum_{h= 0}^{|\states|-1}p_h(s, \mc).
\]

\end{lemma}
\begin{proof}
Consider a sequence of $L+1$ states $T = (s_0, s_1, \ldots, s_{L}) \in \states^{L + 1}$ with $L \ge |\states|$ and $s_L = s$.
By Lemma~\ref{lem:reduction}, there exists another sequence of $L'$ states $T' = (s_0', s_1', \ldots, s_{L'-1}') \in \states^{L'}$ with $s_{L' - 1}' = s_{L} = s$ and $L' \le |\states|$ so that $p(T, \mc) \le p(T', \mc)$.
Therefore \[p(T, \mc) \le p(T', \mc) \le p_{L'-1}(s, \mc) \le \sum_{h = 0}^{|\states|-1}p_{h}(s, \mc),\]
which implies 
\[
p_L(s, \mc) = \sum_{(s_0, s_1, \ldots, s_{L-2}, s_{L-1}) \in \states^{L}}p((s_0, s_1, \ldots, s_{L-2}, s_{L-1}, s), \mc) 
\le |\states|^{L}\sum_{h= 0}^{|\states|-1}p_{h}(s, \mc).\]
Therefore,
\[
\sum_{h = 0}^{4|\states|-1} p_h(s, \mc) \le 4 \cdot |\states| \cdot |\states|^{4 |\states| - 1} \cdot \sum_{h = 0}^{|\states|-1}p_h(s, \mc)  = 4 \cdot |\states|^{4|\states|} \cdot \sum_{h = 0}^{|\states|-1}p_h(s, \mc).
\]
\end{proof}

By applying Lemma~\ref{lem:sum_unexpanded} in a Markov chain $\mc'$ with modified initial state distribution and transition operator, we can also prove that $\sum_{h = 0}^{4|\states|-1} p_{\beta h + \alpha}(s, \mc) \le \exp\left(O(|\states| \log |\states|)\right)\cdot \sum_{h = 0}^{|\states|-1}p_{\beta h + \alpha}(s, \mc).
$ for any integer $\alpha \ge 0$ and integer $\beta \ge 1$. 
\begin{lemma}\label{lem:sum_expanded}
Let $\mc = (\states, \trans, \mu)$ be a Markov chain.
For any integer $\alpha \ge 0 $ and integer $\beta \ge 1$, 
for any $s \in \states$, 
\[
\sum_{h = 0}^{4|\states|-1} p_{\beta h + \alpha}(s, \mc) \le  4 \cdot |\states|^{4 |\states|} \cdot \sum_{h = 0}^{|\states|-1}p_{\beta h + \alpha}(s, \mc).
\]
\end{lemma}
\begin{proof}
We define a new Markov chain $\mc' = (\states, \trans', \mu')$ based on $\mc = (\states, \trans, \mu)$.
The state space of $\mc'$ is the same as that of $\mc$.
The initial state distribution $\mu'$ is the same as the distribution of $s_{\alpha}$ in $\mc$, i.e., the distribution after taking $\alpha$ steps in $\mc$.
The transition operator is defined so that taking one step in $\mc'$ is equivalent to taking $\beta$ steps in $\mc$, i.e., 
\[
P'(s' \mid s) = \sum_{s_1, s_2, \ldots, s_{\beta - 1} \in \states^{\beta - 1}} P(s' \mid s_{\beta - 1}) \cdot P(s_{\beta - 1} \mid s_{\beta -2})  \cdot \cdots \cdot P(s_1 \mid s).
\]
Clearly, for any state $s \in \states$, $p_L(s, \mc') = p_{\beta L + \alpha}(s, \mc)$.
By using Lemma~\ref{lem:sum_unexpanded} in $\mc'$, for any $s \in \states$, we have
\[
\sum_{h = 0}^{4|\states|-1} p_h(s, \mc') \le 4 \cdot |\states|^{4 |\states|} \cdot \sum_{h = 0}^{|\states|-1}p_h(s, \mc'),
\]
which implies
\[
\sum_{h = 0}^{4|\states|-1} p_{\beta h + \alpha}(s, \mc) \le 4 \cdot |\states|^{4 |\states|} \cdot \sum_{h = 0}^{|\states|-1}p_{\beta h + \alpha} (s, \mc).
\]
\end{proof}
Finally, Lemma~\ref{lem:sum_expanded} implies the main result of this section, which shows that for any $L \ge |\states|$, $\sum_{h = 0}^{4L-1} p_h(s, \mc) \le \exp\left(O(|\states| \log |\states|)\right) \cdot \sum_{h= 0}^{L-1}p_h(s, \mc)$.

\begin{lemma}\label{lem:mc_main}
Let $\mc = (\states, \trans, \mu)$ be a Markov chain.
For any $s \in \states$ and $L \ge |\states|$, \[\sum_{h = 0}^{2L} p_h(s, \mc) \le 4 \cdot |\states|^{4 |\states|} \cdot \sum_{h=0}^{L-1} p_h(s, \mc).\]
\end{lemma}
\begin{proof}
Clearly,
\[
\sum_{h=0}^{L-1} p_h(s, \mc) \ge \sum_{h=0}^{\floor{L / |\states|} \cdot |\states|-1} p_h(s, \mc) = \sum_{i = 0}^{\floor{L / |\states|} - 1} \sum_{j = 0}^{|\states| - 1} p_{\floor{L / |\states|} \cdot j + i}(s, \mc).
\]
For each $0 \le i < \floor{L / |\states|}$, by Lemma~\ref{lem:sum_expanded}, we have
\[
\sum_{j = 0}^{|\states| - 1} p_{\floor{L / |\states|} \cdot j + i}(s, C) \ge \frac{1}{4|\states|^{4 |\states|}}\sum_{j = 0}^{4|\states| - 1} p_{\floor{L / |\states|} \cdot j + i}(s, C).
\]

On the other hand,
\[
\sum_{h = 0}^{2L} p_h(s, C) \le \sum_{i = 0}^{\floor{L / |\states|} - 1} \sum_{j = 0}^{\floor{(2L + 1) / \floor{L / |\states|} }-1} p_{\floor{L / |\states|} \cdot j + i}(s, C).
\]

Note that if $|\states| > L / 2$, then $\floor{(2L + 1) / \floor{L / |\states|} } = 2L + 1 < 4|\states|$.
Moreover, if $|\states| \le L / 2$, then we have $\floor{L / |\states|} \ge 2L / 3|\states|$, which implies 
\[
\floor{(2L + 1) / \floor{L / |\states|} } \le \floor{(2L + 1) / L \cdot 3|\states| / 2} \le \floor{4|\states|}  = 4|\states|.
\]
Hence, we have
\[
\sum_{h = 0}^{2L} p_h(s, C) \le  \sum_{i = 0}^{\floor{L / |\states|} - 1} \ \sum_{j = 0}^{4|\states| - 1} p_{\floor{L / |\states|} \cdot j + i}(s, C) \le 4 \cdot |\states|^{4 |\states|} \cdot \sum_{h = 0}^{L-1} p_h(s, C).
\]
\end{proof}

\subsection{Implications of Lemma~\ref{lem:mc_main}}\label{sec:stationary}
In this section, we list several implications of Lemma~\ref{lem:mc_main} which would be crucial for analysis in later sections. 

Our first lemma shows that for any MDP $\mdp$ and any stationary policy $\pi$, for a properly chosen discount factor $\gamma$, $V^{\pi}_{\mdp, \gamma}$ is a multiplicative approximation to $V^{\pi}_{\mdp, H}$ with approximation ratio $\exp\left(O(|\states| \log |\states|)\right)$. 
\begin{lemma}\label{lem:discount_finite_compare}
For any MDP $\mdp$ and any stationary policy $\pi$, if $H \ge 2\ln(8 \cdot |\states|^{4 |\states|})$, by taking $\gamma = 1- \frac{\ln(8 \cdot |\states|^{4 |\states|})}{H}$,
\[
\frac{1}{64 \cdot |\states|^{8|\states|}} V^{\pi}_{\mdp, H} \le V^{\pi}_{\mdp, \gamma} \le 2V^{\pi}_{\mdp, H}.
\]
\end{lemma}
\begin{proof}
\begin{align*}
& V^{\pi}_{\mdp, \gamma} = \sum_{s \in \states} \sum_{h = 0}^{\infty} \gamma^h  \cdot p_h(s, \mdp^{\pi}) \cdot \expect[R(s, \pi(s))]\\
 \le& \sum_{s \in \states} \left( \sum_{h = 0}^{H - 1} p_h(s, \mdp^\pi)+ \sum_{i = 1}^{\infty} \gamma^{H\cdot 2^{i-1}} \left( \sum_{h = 0}^{2^i \cdot H - 1} p_h(s, \mdp^\pi) \right)\right) \cdot \expect[R(s, \pi(s))].
\end{align*}
For each $i \ge 1$, by Lemma~\ref{lem:mc_main}, for any $s \in \states$, 
\begin{align*}
\gamma^{H\cdot 2^{i-1}} \left( \sum_{h = 0}^{2^i \cdot H - 1} p_h(s, \mdp^\pi) \right) & \le \gamma^{H\cdot 2^{i-1}}  \cdot \left( 4 \cdot |\states|^{4 |\states|} \right)^i  \cdot \left(\sum_{h = 0}^{H - 1} p_h(s, \mdp^\pi) \right)\\
& \le \left( 8 \cdot |\states|^{4 |\states|} \right)^{-2^{i-1}} \cdot \left( 4 \cdot |\states|^{4 |\states|} \right)^i \cdot \left(\sum_{h = 0}^{H - 1} p_h(s, \mdp^\pi)  \right)\\
 &\le 1/2^i \cdot \left(\sum_{h = 0}^{H - 1} p_h(s, \mdp^\pi) \right).
 \end{align*}
 Therefore,
 \[
 V^{\pi}_{\mdp, \gamma}  \le \sum_{s \in \states} 2 \cdot \left( \sum_{h = 0}^{H - 1} p_h(s, \mdp^\pi) \right) \cdot \expect[R(s, \pi(s))] = 2V^{\pi}_{\mdp, H}.
 \]

On the other hand, 
\begin{align*}
 V^{\pi}_{\mdp, \gamma} & = \sum_{s \in \states} \sum_{h = 0}^{\infty} \gamma^h \cdot p_h(s, \mdp^\pi)  \cdot \expect[R(s, \pi(s))]  \\
 & \ge  \sum_{s \in \states} \sum_{h = 0}^{H - 1}  \gamma^h \cdot p_h(s, \mdp^\pi)  \cdot \expect[R(s, \pi(s))] \\
 & \ge \gamma^H \cdot \sum_{s \in \states} \sum_{h = 0}^{H - 1}   p_h(s, \mdp^\pi)  \cdot \expect[R(s, \pi(s))] \\
 &  = \gamma^H \cdot V^{\pi}_{\mdp, H} = \left(1- \frac{\ln(8 \cdot |\states|^{4 |\states|})}{H}\right)^H \cdot V^{\pi}_{\mdp, H}  \\
 & \ge (1/4)^{\ln(8 \cdot |\states|^{4 |\states|})}\cdot V^{\pi}_{\mdp, H}  \ge \frac{1}{64 \cdot |\states|^{8|\states|}} \cdot V^{\pi}_{\mdp, H}.
\end{align*}

\end{proof}

As another implication of Lemma~\ref{lem:mc_main}, for any MDP $\mdp$ and any stationary policy $\pi$, we have $V^{\pi}_{\mdp, \floor{H / 2}} \ge  \exp\left( - O(|\states| \log |\states|)\right) V^{\pi}_{\mdp, H}.$
\begin{lemma}\label{lem:half_trajectory}
For any MDP $\mdp$ and any stationary policy $\pi : \states \to \actions$, if $H \ge 2 |\states|$,
\[
V^{\pi}_{\mdp, \floor{H / 2}} \ge  \frac{1}{4 \cdot |\states|^{4 |\states|}}  V^{\pi}_{\mdp, H}.
\]
\end{lemma}
\begin{proof}
Note that 
\[
V^{\pi}_{\mdp, \floor{H / 2}} = \sum_{s \in \states} \sum_{h = 0}^{\floor{H / 2} - 1} p_h(s, \mdp^\pi)  \cdot \expect[R(s, \pi(s))].
\]
Since $H \ge 2 |\states|$, by Lemma~\ref{lem:mc_main}, for any $s \in \states$,
\[
\sum_{h = 0}^{\floor{H / 2} - 1} p_h(s, \mdp^\pi)  \ge \frac{1}{4 \cdot |\states|^{4 |\states|}} \sum_{h = 0}^{2 \floor{H / 2} } p_h(s, \mdp^\pi)   \ge \frac{1}{4 \cdot |\states|^{4 |\states|}} \sum_{h = 0}^{H - 1} p_h(s, \mdp^\pi) .
\]
Therefore,
\[
V^{\pi}_{\mdp, \floor{H / 2}} \ge  \frac{1}{4 \cdot |\states|^{4 |\states|}}   \sum_{s \in \states} \sum_{h = 0}^{H - 1} p_h(s, \mdp^\pi)  \cdot \expect[R(s, \pi(s))] =  \frac{1}{4 \cdot |\states|^{4 |\states|}}  
V^{\pi}_{\mdp, H}.
\]
\end{proof}

As a corollary of Lemma~\ref{lem:discount_finite_compare}, we show that for any $H$-horizon MDP $\mdp$, there always exists a stationary policy whose value is as large as the best non-stationary policy up to a factor of $\exp\left( O(|\states| \log |\states|)\right)$.
\begin{corollary}\label{cor:non_stationary_nearly_optimal}
For any MDP $\mdp$, if $H \ge 2\ln(8 \cdot |\states|^{4 |\states|})$, then there exists a stationary policy $\pi$ such that
\[
V^{\pi}_{\mdp, H} \ge \frac{1}{128 \cdot |\states|^{8|\states|}}  V^{\pi^*}_{\mdp, H},
\]
\end{corollary}
\begin{proof}
In this proof we fix $\gamma = 1- \frac{\ln(8 \cdot |\states|^{4 |\states|})}{H}$.
We also use $\widetilde{\pi}^*$ to denote a non-stationary policy such that $\widetilde{\pi}^*_h = \pi^*_h$ when $h \in [H]$ and $\widetilde{\pi}^*_h$ is defined arbitrarily when $h \ge H$. 

Clearly, there exists a stationary policy $\pi$ such that for any (possibly non-stationary) policy $\pi'$,
\[
V^{\pi}_{\mdp, \gamma} \ge V^{\pi'}_{\mdp, \gamma}.
\]
For a proof, see Theorem 5.5.3 in~\cite{puterman1994markov}.
Clearly,
\[V^{\widetilde{\pi}^*}_{\mdp, \gamma} \ge \gamma^H \cdot V^{\pi^*}_{\mdp, H}.\]
Moreover, by Lemma~\ref{lem:discount_finite_compare}, 
\[
V^{\pi}_{\mdp, H} \ge \frac{1}{2} V^{\pi}_{\mdp, \gamma} \ge \frac{1}{2} V^{\widetilde{\pi}^*}_{\mdp, \gamma} \ge \frac{1}{2} \cdot \gamma^H \cdot V^{\pi^*}_{\mdp, H} \ge \frac{1}{128 \cdot |\states|^{8|\states|}}  \cdot V^{\pi^*}_{\mdp, H}.
\]
\end{proof}
By applying Corollary~\ref{cor:non_stationary_nearly_optimal} in an MDP with an extra terminal state $s_{\mathrm{terminal}} $, we can show that for any $(s, a) \in \states \times \actions$, there always exists a stationary policy that visits $(s, a)$ in the first $H / 2$ time steps with probability as large as the probability that the best non-stationary policy visits $(s, a)$ in all the $H$ time steps, up to a factor of $\exp \left( O(|\states| \log |\states|)\right)$.
\begin{corollary}\label{cor:reaching_stationary}
For any MDP $\mdp$, if $H \ge 2\ln(8 \cdot (|\states| + 1)^{4 (|\states| + 1)})$, then for any $z \in \states \times \actions$, there exists a stationary policy $\pi$, such that for any (possibly non-stationary) policy $\pi'$,
\[
\Pr\left[\sum_{h = 0}^{\floor{H / 2} - 1} \indict[(s_h, a_h) = z] \ge 1  \right] \ge  \frac{1}{512 \cdot (|\states| + 1)^{12(|\states| + 1)}}  \Pr\left[\sum_{h = 0}^{H - 1} \indict[(s_h', a_h') = z] \ge 1 \right],
\]
where 
\[
(s_0, a_0), (s_1, a_1), \ldots, (s_{H - 1}, a_{H - 1}), s_H
\]
is a random trajectory induced by executing $\pi$ in $\mdp$ and
\[
(s_0', a_0'), (s_1', a_1'), \ldots, (s_{H - 1}', a_{H - 1}'), s_H'
\]
is a random trajectory induced by executing $\pi'$ in $\mdp$.
\end{corollary}
\begin{proof}
For the given MDP $\mdp =\left(\states, \actions, \trans ,R, H, \mu\right)$, we create a new MDP \[\mdp' = \left(\states \cup \{s_{\mathrm{terminal}}\}, \actions, \trans' , R', H, \mu\right),\]
where $s_{\mathrm{terminal}}$ is a state such that $s_{\mathrm{terminal}} \notin \states$. Moreover, 
\[
\trans'(s, a) = \begin{cases}
\trans(s, a) & s \neq s_{\mathrm{terminal}} \text{ and } (s, a) \neq z \\
s_{\mathrm{terminal}} & s = s_{\mathrm{terminal}} \text{ or } (s, a) = z 
\end{cases}
\]
and
\[
R'(s, a) = \indict[(s, a) = z].
\]
Clearly, for any policy $\pi$,
\[
V^{\pi}_{\mdp', H} = \Pr\left[\sum_{h = 0}^{H - 1} \indict[(s_h, a_h) = (s, a)] \ge 1 \right]
\]
where
\[
(s_0, a_0), (s_1, a_1), \ldots, (s_{H - 1}, a_{H - 1}), s_H
\]
is a random trajectory induced by executing $\pi$ in $\mdp$.
Therefore, by Corollary~\ref{cor:non_stationary_nearly_optimal}, there exists a stationary policy $\pi$ such that for any (possibly non-stationary) policy $\pi'$,
\[
V^{\pi}_{\mdp', H} \ge  \frac{1}{128 \cdot (|\states| + 1)^{8(|\states| + 1)}}  V^{\pi'}_{\mdp', H} .
\]
Moreover, by Lemma~\ref{lem:half_trajectory}, for any (possibly non-stationary) policy $\pi'$,
\[
V^{\pi}_{\mdp', \floor{H / 2}} \ge  \frac{1}{512 \cdot (|\states| + 1)^{12(|\states| + 1)}}  V^{\pi'}_{\mdp', H},
\]
which implies the desired result. 
\end{proof}
Finally, by combining Lemma~\ref{lem:half_trajectory} and Corollary~\ref{cor:non_stationary_nearly_optimal}, we can show that for any $(s, a) \in \states \times \actions$, if the initial state distribution $\mu$ is $s$ and there exists a non-stationary policy that visits $(s, a)$ for $f$ times with probability $\epsilon$ in all the $H$ steps, then there exists a stationary policy that visits $(s, a)$ for $\exp\left(-O(|\states |\log |\states|)\right) \cdot \epsilon \cdot f$ times with constant probability in the first $H / 2$ steps. 
\begin{corollary}\label{cor:quantile_comparison}
For a given MDP $\mdp$ and a state-action pair $z = (s_z, a_z) \in \states \times \actions$, suppose the initial state distribution $\mu$ is $s_z$ and $H \ge 2\ln(8 \cdot |\states|^{4 |\states|})$.
If there exists a (possibly non-stationary) policy $\pi'$ such that $\qtile^{\pi'}_{\epsilon}(s_z, a_z) \ge f$ for some integer $0 \le f \le H$,
then there exists a stationary policy $\pi$ such that 
\[
\qtile^{\pi}_{1/2}\left( \sum_{h = 0}^{\floor{H / 2} - 1} \indict[(s_h, a_h) = z]\right) \ge \left \lfloor \frac{1}{2048 \cdot |\states|^{12|\states|}} \cdot \epsilon \cdot f \right \rfloor
\]
where
\[
(s_0, a_0, r_0), (s_1, a_1, r_1) \ldots, (s_{H - 1}, a_{H - 1}, r_{H - 1}), s_H
\]
is a trajectory induced by executing $\pi$ in $\mdp$. 
\end{corollary}
\begin{proof}
If $f = 0$ then the lemma is clearly true.
No consider the case $f > 0$.
Consider a new MDP $\mdp' = \left(\states, \actions, \trans , R', H, \mu\right)$ where $R(s, a) = \indict[(s, a) = z]$.
Clearly, $V^{\pi'}_{\mdp', H} \ge \epsilon \cdot f$.
By Corollary~\ref{cor:non_stationary_nearly_optimal}, there exists a stationary policy $\pi$ such that 
\[
V^{\pi}_{\mdp', H} \ge   \frac{1}{128 \cdot |\states|^{8|\states|}}  \cdot \epsilon \cdot f.
\]
By Lemma~\ref{lem:half_trajectory},
\[
V^{\pi}_{\mdp', \floor{H / 2}} \ge   \frac{1}{512 \cdot |\states|^{12|\states|}}  \cdot \epsilon \cdot f.
\]
This implies $\pi(s_z) = a_z$.

Now we use $X$ to denote a random variable which is defined to be
\[
X = \min\{h \ge 1 \mid (s_h, a_h) = z\}.
\]
Here the trajectory 
\[
(s_0, a_0), (s_1, a_1), \ldots 
\] is induced by executing the stationary policy $\pi$ in $\mdp'$.
We also write $\hat{X} = \min\{\floor{H / 2}, X\}$.
We use $\{X_i\}_{i = 1}^{\infty}$ to denote a sequence of i.i.d. copies of $\hat{X}$.
We use $\tau$ to denote a random variable which is defined to be
\[
\tau = \min\left\{i \ge 1 \mid \sum_{j = 1}^i X_j\ge \floor{H / 2}\right\}.
\] 
Clearly, $\tau \le H / 2$ almost surely. Moreover, $\pi$ is a stationary policy, the initial state distribution $\mu = s_z$ deterministically and $\pi(s_z) = a_z$, which implies $\tau$ and $\sum_{h = 0}^{\floor{H / 2}- 1} \indict[(s_h, a_h) = z]$ have the same distribution. 
Indeed, whenever the trajectory $(s_0, a_0), (s_1, a_1) \ldots$ visits $z$, it corresponds to restart a new copy of $\hat{X}$. 

Now for each $i > 0$, we define $Y_i = X_i - \expect[\hat{X}]$. Clearly $\expect[Y_i] = 0$. 
Let $S_i = 0$ and $S_i = \sum_{j = 1}^i Y_j$ for all $i > 0$.
Clearly $\tau$ is a stopping time, and \[
\sum_{j = 1}^\tau X_j \le H
\]
 since $X_i \le \floor{H / 2}$ for all $i > 0$.
By Martingale Stopping Theorem, we have \[
\expect[S_\tau] = \sum_{j = 1}^\tau \expect[X_j] - \expect[\tau] \cdot \expect[\hat{X}] = 0,
\]
 which implies $\expect[\tau] \cdot  \expect[\hat{X}] \le H$ and therefore \[ \expect[\hat{X}] \le H / \expect[\tau] = H / V^{\pi}_{\mdp', \floor{H / 2}} \le 512 \cdot |\states|^{12|\states|} H / (\epsilon \cdot f),\]
where we use the fact that \[V^{\pi}_{\mdp', \floor{H / 2}}=\expect\left[\sum_{h = 0}^{\floor{H / 2}- 1} \indict[(s_h, a_h) = z]\right]=\expect[\tau].\]

Let $\tau' = \left \lfloor \frac{1}{2048 \cdot |\states|^{12|\states|}} \cdot \epsilon \cdot f \right \rfloor$. By Markov's inequality, with probability at least $1/2$, 
\[
\sum_{i = 1}^{\tau'} X_i \le 2 \tau' \expect[\hat{X}] \le H / 2,
\]
in which case $\tau\ge\tau'$. Consequently, 
\[
\qtile^{\pi}_{1/2}\left( \sum_{h = 0}^{\floor{H / 2} - 1} \indict[(s_h, a_h) = z]\right)  = \qtile_{1/2}(\tau) \ge \left \lfloor \frac{1}{2048 \cdot |\states|^{12|\states|}} \cdot \epsilon \cdot f \right \rfloor.\]
\end{proof}

%% file: collect.tex
\subsection{Collecting Samples} \label{sec:collect}
In this section, we present our algorithm for collecting samples. The algorithm is formally presented in Algorithm~\ref{alg:collect}.
The dataset $D$ returned by Algorithm~\ref{alg:collect} consists of $N$ lists, where for each list, elements in the list are tuples of the form $(s, a, r, s') \in \states \times \actions \times [0, 1] \times \states$.
To construct these lists, Algorithm~\ref{alg:collect} enumerates a state-action pair $(s, a) \in \states \times \actions$ and a pair of stationary policies $(\pi_1, \pi_2)$, and then collects a trajectory using $\pi_1$ and $\pi_2$.
More specifically, $\pi_1$ is executed until the trajectory visits $(s,a)$, at which point $\pi_2$ is executed until the last step. 

\begin{algorithm}[!htb]
  \caption{Collect Samples}\label{alg:collect}
  \begin{algorithmic}[1]
  \State \textbf{Input:} number of repetitions $N$
  \State \textbf{Output:} Dataset $D$ where $D= \left( \left(\left( s_{i, t}, a_{i, t}, r_{i, t}, s'_{i, t}\right)\right)_{t = 0}^{|\states||\actions| \cdot |\actions|^{2|\states|} \cdot H - 1} \right)_{i = 0}^{N - 1}$
      \For{$i \in [N]$}
      \State Let $T_i$ be an empty list
  \For{$(s, a) \in \states \times \actions$}  \label{line:enumerate_sa}
  \For{$(\pi_1, \pi_2) \in \Pi_{\sta} \times \Pi_{\sta}$} \label{line:enumerate_policies}

  \State Receive $s_0\sim \mu$
  \For{$h \in [H]$}
  \If{$(s, a) = (s_{h'}, a_{h'})$ for some $h' < h$}
  \State Take $a_h= \pi_2(s_h)$
  \Else
  \State Take $a_h = \pi_1(s_h)$
  \EndIf
  \State Receive $r_h \sim R(s_h, a_h)$ and $s_{h + 1}\sim P(s_h, a_h)$ 

  \State Append $(s_h, a_h, r_h, s_{h + 1})$ to the end of $T_i$
    \EndFor
  \EndFor
  \EndFor
  \EndFor
  \State \Return{$D$} where $D = (T_i)_{i = 0}^{N - 1}$
  \end{algorithmic}
\end{algorithm}

Throughout this section, we use $\mdp =\left(\states, \actions, \trans ,R, H, \mu\right)$ to denote the underlying MDP that the agent interacts with. 
For each $(s, a) \in \states \times \actions$ and $(\pi_1, \pi_2) \in  \Pi_{\sta} \times \Pi_{\sta}$, let \[
(s_0^{s, a, \pi_1, \pi_2}, a_0^{s, a, \pi_1, \pi_2}, r_0^{s, a, \pi_1, \pi_2}), (s_1^{s, a, \pi_1, \pi_2}, a_1^{s, a, \pi_1, \pi_2}, r_1^{s, a, \pi_1, \pi_2}), \ldots, (s_{H - 1}^{s, a, \pi_1, \pi_2}, a_{H - 1}^{s, a, \pi_1, \pi_2}, r_{H - 1}^{s, a, \pi_1, \pi_2}), s_H^{s, a, \pi_1, \pi_2}
\]
by a trajectory where $s_0^{s, a, \pi_1, \pi_2} \sim \mu$ and $s_h^{s, a, \pi_1, \pi_2} \sim P(s_{h - 1}^{s, a, \pi_1, \pi_2}, a_{h - 1}^{s, a, \pi_1, \pi_2})$ for all $1 \le h \le H$, $r_h^{s, a, \pi_1, \pi_2} \sim R(s_h^{s, a, \pi_1, \pi_2}, a_h^{s, a, \pi_1, \pi_2})$ for all $h \in [H]$, and 
\[
a_h^{s, a, \pi_1, \pi_2} = \begin{cases}
\pi_2(s_h^{s, a, \pi_1, \pi_2}) & (s, a) = (s_{h'}^{s, a, \pi_1, \pi_2}, a_{h'}^{s, a, \pi_1, \pi_2}) \text{ for some $h' < h$}  \\
\pi_1(s_h^{s, a, \pi_1, \pi_2}) & \text{otherwise}
\end{cases}
\]
for all $h \in [H]$. 
Note that the above trajectory is the one collected by Algorithm~\ref{alg:collect} when a specific state-action pair $(s, a)$ and a specific pair of policies $(\pi_1, \pi_2)$ are used. 

For any $\epsilon \in (0, 1]$, define
\[
\qtile^{\sta}_{\epsilon}(s, a) = \qtile_{\epsilon}\left( \sum_{(s', a') \in \states \times \actions }\sum_{\pi_1 \in \Pi_{\sta}}\sum_{\pi_2 \in \Pi_{\sta}}  \sum_{h = 0}^{H - 1} \indict[(s_h^{s', a', \pi_1, \pi_2} , a_h^{s', a', \pi_1, \pi_2} ) = (s, a)]\right).
\]
Clearly, $\qtile^{\sta}_{\epsilon}(s, a)$ is the $\epsilon$-quantile of the frequency that $(s, a)$ appears in each $T_i$.

In Lemma~\ref{lem:stationary_quantile}, we first show that for each $(s, a) \in \states \times \actions$, if there exists a policy $\pi$ that visits $(s, a)$ for $m(s, a)$ times with probability at least $\epsilon$, then 
\[
\qtile^{\sta}_{\epsilon / \exp(O(|\states| \log |\states|))}(s, a) \ge m(s, a) / \exp(O(|\states| \log |\states|)).
\]

\begin{lemma}\label{lem:stationary_quantile}
Let $\epsilon \in (0, 1]$ be a given real number. 
For each $(s, a) \in \states \times \actions$, let $m_{\epsilon}(s, a)$ be the largest integer such that there exists a (possibly non-stationary) policy $\pi_{s,a}$ such that 
$\qtile^{\pi_{s, a}}_\epsilon(s, a) \ge m_{\epsilon}(s, a)$.
Then for each $(s, a) \in \states \times \actions$,
\[
\qtile^{\sta}_{ \epsilon (|\states| + 1)^{-12(|\states| + 1) }/ 1024} (s, a)\ge  \frac{1}{4096 \cdot |\states|^{12|\states|}} \cdot \epsilon \cdot m_{\epsilon}(s, a)  .
\]
\end{lemma}
\begin{proof}

For each $(s, a) \in \states \times \actions$,
there exists a (possibly non-stationary) policy $\pi_{s, a}$ such that $\qtile^{\pi_{s, a}}_\epsilon(s, a) \ge m_{\epsilon}(s, a)$.
Here we consider the case that $m_{\epsilon}(s, a) \ge 1$, since otherwise the lemma clearly holds. 
By Corollary~\ref{cor:reaching_stationary}, there exists a stationary policy $\pi_{s, a}'$ such that
\[
\Pr\left[\sum_{h = 0}^{\floor{H / 2} - 1} \indict[(s_h, a_h) = (s, a)] \ge 1  \right] \ge \frac{ \epsilon}{512 \cdot (|\states| + 1)^{12(|\states| + 1)}},
\]
where
\[
(s_0, a_0), (s_1, a_1), \ldots, (s_{H - 1}, a_{H - 1}), s_H
\]
is a random trajectory induced by executing $\pi_{s, a}'$ in $\mdp$.

In the remaining part of the analysis, we consider two cases.
\paragraph{Case I: $m_{\epsilon}(s, a) \ge 4096 \cdot |\states|^{12|\states|} / \epsilon$.}
Let \[(s_0, a_0), (s_1, a_1), \ldots, (s_{H - 1}, a_{H - 1}), s_H\] be a random trajectory induced by executing $\pi_{s, a}$ in $\mdp$.
Let $X_{s, a}$ be the random variable which is defined to be 
\[
X_{s, a}= \begin{cases}
\min\{h \in [H] \mid (s_h, a_h) = (s, a)\} & \text{if there exists $h \in [H]$ such that $(s_h, a_h) = (s, a)$}\\
H & \text{otherwise}
\end{cases}.
\]
Clearly,
\begin{align*}
&  \sum_{h' = 0}^{H - 1} \Pr[X_{s, a} = h'] \cdot \Pr\left[\sum_{h = h'}^{H - 1} \indict[(s_h, a_h) = (s, a)] \ge m_{\epsilon}(s, a) \mid (s_{h'}, a_{h'}) = (s, a) \right] \\
= &  \Pr\left[\sum_{h = 0}^{H - 1} \indict[(s_h, a_h) = (s, a)] \ge m_{\epsilon}(s, a) \right] \ge \epsilon. 
\end{align*}
Therefore, there exists $h' \in [H]$ such that
\[
\Pr[X_{s, a} = h'] > 0
\]
and 
\[
\Pr\left[\sum_{h = h'}^{H - 1} \indict[(s_h, a_h) = (s, a)] \ge m_{\epsilon}(s, a) \mid (s_{h'}, a_{h'}) = (s, a) \right] \ge \epsilon. 
\]
Note that we must have $\pi_{h'}(s) = a$, since otherwise $\Pr[X_{s, a} = h'] = 0$.

Now we consider a new MDP $\mdp_{s}= \left(\states, \actions, \trans , R, H, \mu_{s} \right)$ where $\mu_{s} = s$. 
Let $\widetilde{\pi}$ be an arbitrary policy so that $\widetilde{\pi}_{h} = (\pi_{s, a})_{h' + h}$ for all $h \in [H - h']$. 
Clearly,
\[
\Pr\left[\sum_{h = 0}^{H - 1} \indict[(s_h', a_h') = (s, a)] \ge m_{\epsilon}(s, a) \right]  \ge \Pr\left[\sum_{h = 0}^{H - h' - 1} \indict[(s_h', a_h') = (s, a)] \ge m_{\epsilon}(s, a) \right] \ge \epsilon
\]
where
\[
(s_0', a_0'), (s_1', a_1'), \ldots, (s_{H - 1}', a_{H - 1}'), s_H'
\]
is a random trajectory induced by executing $ \widetilde{\pi}$ in $\mdp_{s}$.
Therefore, by Corollary~\ref{cor:quantile_comparison},
there exists a stationary policy $\widetilde{\pi}_{s, a}$ such that 
\[
\Pr\left[\sum_{h = 0}^{\floor{H / 2}- 1} \indict[(s_h'', a_h'') = (s, a)] \ge \left \lfloor \frac{1}{2048 \cdot |\states|^{12|\states|}} \cdot \epsilon \cdot m_{\epsilon}(s, a) \right \rfloor \right] \ge 1/2
\]
where
\[
(s_0'', a_0''), (s_1'', a_1''), \ldots, (s_{H - 1}'', a_{H - 1}''), s_H''
\]
is a random trajectory induced by executing $\widetilde{\pi}_{s, a}$ in $\mdp_{s}$.
Since $m_{\epsilon}(s, a) \ge 4096 \cdot |\states|^{12|\states|} / \epsilon$ and thus $\left \lfloor \frac{1}{2048 \cdot |\states|^{12|\states|}} \cdot \epsilon \cdot m_{\epsilon}(s, a) \right \rfloor\ge 1$, we must have $\widetilde{\pi}_{s, a}(s) = a$.

Now we consider the case when $\pi_1 = \pi_{s, a}'$ and $\pi_2 = \widetilde{\pi}_{s, a}$.
Since $\pi_1 = \pi_{s, a}'$,
\[
\Pr\left[\sum_{h = 0}^{\floor{H / 2} - 1} \indict[(s_h^{s, a, \pi_1, \pi_2}, a_h^{s, a, \pi_1, \pi_2}) = (s, a)] \ge 1 \right] \ge \frac{ \epsilon}{512 \cdot (|\states| + 1)^{12(|\states| + 1)}}  .
\]
Therefore, let $X_{s, a}'$ be the random variable which is defined to be 
\[
X_{s, a}'= \begin{cases}
\min\{h \in [\floor{H / 2}] \mid (s_h^{s, a, \pi_1, \pi_2}, a_h^{s, a, \pi_1, \pi_2}) = (s, a)\} & \text{if $(s_h, a_h) = (s, a)$ for some $h \in [\floor{H / 2}]$}\\
\floor{H / 2} & \text{otherwise}
\end{cases}.
\]
We have that
\[
\Pr[X_{s, a}'  \in [\floor{H / 2}] ] \ge \frac{ \epsilon}{512 \cdot (|\states| + 1)^{12(|\states| + 1)}}  .
\]
Moreover, for each $h' \in [\floor{H / 2}]$, since $\pi_2 = \widetilde{\pi}_{s, a}$, 
\begin{align*}
& \Pr\left [\sum_{h = h'}^{H - 1} \indict[(s_h^{s, a, \pi_1, \pi_2}, a_h^{s, a, \pi_1, \pi_2}) = (s, a)] \ge \left \lfloor \frac{1}{2048 \cdot |\states|^{12|\states|}} \cdot \epsilon \cdot m_{\epsilon}(s, a) \right \rfloor \mid (s_{h'}^{s, a, \pi_1, \pi_2}, a_{h'}^{s, a, \pi_1, \pi_2}) = (s, a) \right] \\
\ge & 1/2. 
\end{align*}
Therefore,
\begin{align*}
 & \Pr\left [\sum_{h = 0}^{H - 1} \indict[(s_h^{s, a, \pi_1, \pi_2}, a_h^{s, a, \pi_1, \pi_2}) = (s, a)] \ge \left \lfloor \frac{1}{2048 \cdot |\states|^{12|\states|}} \cdot \epsilon \cdot m_{\epsilon}(s, a) \right \rfloor \right] \\
 \ge & \sum_{h ' = 0}^{\floor{H / 2} - 1} \Pr[X_{s, a}' = h'] \\
 \cdot & \Pr\left [\sum_{h = h'}^{H - 1} \indict[(s_h^{s, a, \pi_1, \pi_2}, a_h^{s, a, \pi_1, \pi_2}) = (s, a)] \ge \left \lfloor \frac{1}{2048 \cdot |\states|^{12|\states|}} \cdot \epsilon \cdot m_{\epsilon}(s, a) \right \rfloor \mid (s_{h'}^{\pi_1, \pi_2}, a_{h'}^{\pi_1, \pi_2}) = (s, a) \right]\\
 \ge& \frac{ \epsilon}{1024 \cdot (|\states| + 1)^{12(|\states| + 1)}}  .
\end{align*}
Since $m_{\epsilon}(s, a) \ge 4096 \cdot |\states|^{12|\states|} / \epsilon$, we have
\[
\Pr\left [\sum_{h = 0}^{H - 1} \indict[(s_h^{s, a, \pi_1, \pi_2} a_h^{s,a, \pi_1, \pi_2}) = (s, a)] \ge  \frac{1}{4096 \cdot |\states|^{12|\states|}} \cdot \epsilon \cdot m_{\epsilon}(s, a) \right] \ge \frac{ \epsilon}{1024 \cdot (|\states| + 1)^{12(|\states| + 1)}}
\]
and thus
\[
\qtile^{\sta}_{ \epsilon (|\states| + 1)^{-12(|\states| + 1) }/ 1024} (s, a)\ge  \frac{1}{4096 \cdot |\states|^{12|\states|}} \cdot \epsilon \cdot m_{\epsilon}(s, a)  .
\]
%

\paragraph{Case II: $m_{\epsilon}(s, a) < 4096 \cdot |\states|^{12|\states|} / \epsilon$.}
Consider the case when $\pi_1 = \pi_2 = \pi_{s, a}'$.
Clearly,
\begin{align*}
& \Pr\left[\sum_{h = 0}^{H - 1} \indict[(s_h^{s, a, \pi_1, \pi_2}, a_h^{s, a, \pi_1, \pi_2}) = (s, a)] \ge 1 \right] \\
 \ge&  \Pr\left[\sum_{h = 0}^{\floor{H / 2} - 1} \indict[(s_h^{s, a, \pi_1, \pi_2}, a_h^{s, a, \pi_1, \pi_2}) = (s, a)] \ge 1 \right] \ge \frac{ \epsilon}{512 \cdot (|\states| + 1)^{12(|\states| + 1)}} 
\end{align*}
and thus
\[
\qtile^{\sta}_{ \epsilon (|\states| + 1)^{-12(|\states| + 1) }/ 1024} (s, a)\ge  \frac{1}{4096 \cdot |\states|^{12|\states|}} \cdot \epsilon \cdot m_{\epsilon}(s, a)  .
\]
%
%
\end{proof}

Now we show that for a given percentile $\epsilon$, for the dataset $D$ returned by Algorithm~\ref{alg:collect}, 
for each $(s, a) \in \states \times \actions$, $(s, a)$ appears for at least $\qtile^{\sta}_{\epsilon / 4}(s, a)$ times 
for at least $\Omega(N \cdot \epsilon)$ lists out of the $N$ lists returned by Algorithm~\ref{alg:collect}. 
\begin{lemma}\label{lem:number_exceed_quantile}
Let $\epsilon, \delta \in (0, 1]$ be a given real number.
Let $D$ be the dataset returned by Algorithm~\ref{alg:collect} where 
\[
D= \left( \left(\left( s_{i, t}, a_{i, t}, r_{i, t}, s'_{i, t}\right)\right)_{t = 0}^{|\states||\actions| \cdot |\actions|^{2|\states|} \cdot H - 1} \right)_{i = 0}^{N - 1}.
\]
Suppose $N \ge 16 / \epsilon \cdot \log(3|\states||\actions| / \delta)$. With probability at least $1 - \delta / 3$, for each $(s, a) \in \states \times \actions$, we have
\[
\sum_{i = 0}^{N - 1} \indict\left[ \sum_{t = 0}^{|\states||\actions| \cdot |\actions|^{2|\states|} \cdot H - 1} \indict[(s_{i, t}, a_{i, t}) = (s, a)] \ge \qtile^{\sta}_{\epsilon / 4}(s, a)\right] \ge N \epsilon / 8.
\]
\end{lemma}
\begin{proof}
For each $(s, a) \in \states \times \actions$, 
by the definition of $\qtile^{\sta}_{\epsilon / 4}(s, a)$, for each $i \in [N]$, we have
\[
\expect \left[ \indict\left[ \sum_{t = 0}^{|\states||\actions| \cdot |\actions|^{2|\states|} \cdot H - 1} \indict[(s_{i, t}, a_{i, t}) = (s, a)] \ge \qtile^{\sta}_{\epsilon / 4}(s, a)\right]  \right] \ge \epsilon / 4. 
\]
Hence, the desired result follows by Chernoff bound and a union bound over all $(s, a) \in \states \times \actions$. 
\end{proof}

We also need a subroutine to estimate $\qtile^{\sta}_{\epsilon_{\mathrm{est}}}(s, a)$ for some $\epsilon_{\mathrm{est}}$ to be decided.
Such estimates are crucial for building estimators for the transition probabilities and the rewards with bounded variance, which we elaborate in later parts of this section. 

Our algorithm for estimating $\qtile^{\sta}_{\epsilon_{\mathrm{est}}}(s, a)$ is described in Algorithm~\ref{alg:estimate}.
Algorithm~\ref{alg:estimate} collects $N$ lists, where for each list, elements in the list are tuples of the form $(s, a, r, s') \in \states \times \actions \times [0, 1] \times \states $. These $N$ lists are collected using the same approach as in Algorithm~\ref{alg:collect}. 
Once these $N$ lists are collected, for each $(s, a) \in \states \times \actions$, our estimate (denoted as $\overline{m}^{\sta}(s, a)$) is then set to be the $\ceil{N \cdot \epsilon_{\mathrm{est}} / 2}$-th largest element in $F_{s, a}$, where $F_{s, a}$ is the set of the number of times $(s, a)$ appear in each of the $N$ lists. 

\begin{algorithm}[!htb]
  \caption{Estimate Quantiles}\label{alg:estimate}
  \begin{algorithmic}[1]
  \State \textbf{Input:} Percentile $\epsilon_{\mathrm{est}}$, failure probability $\delta_{\mathrm{est}}$
  \State \textbf{Output:} Estimates $\overline{m}^{\sta} : \states \times \actions \to \mathbb{N}$
  \State Let $N = \ceil{300\log(6|\states||\actions| / \delta_{\mathrm{est}}) / \epsilon_{\mathrm{est}}}$
    \State Let $F_{s, a}$ be an empty multiset for all $(s, a) \in \states \times \actions$
    \For{$i \in [N]$}
    \State Let $T_i$ be an empty list
    \For{$(s, a) \in \states \times \actions$}
  \For{$(\pi_1, \pi_2) \in \Pi_{\sta} \times \Pi_{\sta}$} 

  \State Receive $s_0 \sim \mu$
  \For{$h \in [H]$}
  \If{$(s, a) = (s_{h'}, a_{h'})$ for some $0 \le h' < h$}
  \State Take $a_h = \pi_2(s_h)$
  \Else
  \State Take $a_h = \pi_1(s_h)$
  \EndIf
  \State Receive $r_h \sim R(s_h, a_h)$ and $s_{h + 1} \sim P(s_h, a_h)$ 
  \State Append $(s_h, a_h, r_h, s_{h + 1})$ to the end of $T_i$
    \EndFor
  \EndFor
    \EndFor
  \For{$(s, a) \in \states \times \actions$}
  \State Add $\sum_{t = 0}^{|T_i| - 1} \indict[(s_t, a_t) = (s, a)]$ into $F_{s, a}$ where \[T_i = ((s_0, a_0, r_0, s_0'), (s_1, a_1, r_1, s_1'), \ldots, (s_{|T_i| - 1}, a_{|T_i| - 1}, r_{|T_i| - 1}, s_{|T_i| - 1}'))\]
  \EndFor
  \EndFor
    \For{$(s, a) \in \states \times \actions$}
    \State Set $\overline{m}^{\sta}(s, a)$ be the $\ceil{N \cdot \epsilon_{\mathrm{est}} / 2}$-th largest element in $F_{s, a}$
      \EndFor

  \State \Return{$\overline{m}^{\sta}$}
  \end{algorithmic}
 
\end{algorithm}

We now show that for each $(s, a) \in \states \times \actions$, $\overline{m}^{\sta}(s, a)$ is an accurate estimate of $\qtile^{\sta}_{\epsilon_{\mathrm{est}}}(s, a) $.

\begin{lemma}\label{lem:estimate_quantile}
Let $\overline{m}^{\sta}$ be the function returned by Algorithm~\ref{alg:estimate}.
With probability at least $1 - \delta_{\mathrm{est}}/3$, for all $(s, a) \in \states \times \actions$, 
\[
	\qtile^{\sta}_{\epsilon_{\mathrm{est}}}(s, a) \le \overline{m}^{\sta}(s, a) \le \qtile^{\sta}_{\epsilon_{\mathrm{est}} / 4}(s, a).
\]
\end{lemma}
\begin{proof}
Fix a state-action pair $(s, a) \in \states \times \actions$.
For each $i \in [N]$, define
\[
\underline{X}_i = \indict\left[
  \sum_{t = 0}^{|T_i| - 1} \indict[(s_t, a_t) = (s, a)] > \qtile^{\sta}_{\epsilon_{\mathrm{est}} / 4}(s, a) \right]
\]
 where \[T_i = ((s_0, a_0, r_0, s_0'), (s_1, a_1, r_1, s_1'), \ldots, (s_{|T_i| - 1}, a_{|T_i| - 1}, r_{|T_i| - 1}, s_{|T_i| - 1}')).\]

For each $i \in [N]$, by the definition of $ \qtile^{\sta}_{\epsilon_{\mathrm{est}} / 4}(s, a)$, we have $\expect[\underline{X}_i ] \le \epsilon_{\mathrm{est}} / 4$ and thus $\sum_{i = 0}^{N - 1} \expect[\underline{X}_i] \le N \cdot \epsilon_{\mathrm{est}} / 4$.
By Chernoff bound, with probability at most $\delta_{\mathrm{est}} / (6|\states||\actions|)$,
\[
\sum_{i = 0}^{N - 1} \underline{X}_i \ge N \cdot \epsilon_{\mathrm{est}} / 3. 
\]

On the other hand, for each $i \in [N]$, define
\[
\overline{X}_i = \indict\left[ \sum_{t = 0}^{|T_i| - 1} \indict[(s_{i, t} , a_{i, t} ) = (s, a)] \ge \qtile^{\sta}_{\epsilon_{\mathrm{est}}}(s, a) \right]. 
\]
 where \[T_i = ((s_0, a_0, r_0, s_0'), (s_1, a_1, r_1, s_1'), \ldots, (s_{|T_i| - 1}, a_{|T_i| - 1}, r_{|T_i| - 1}, s_{|T_i| - 1}')).\]

For each $i \in [N]$, by the definition of $ \qtile^{\sta}_{\epsilon_{\mathrm{est}}}(s, a) $, we have $\expect[\overline{X}_i] \ge \epsilon_{\mathrm{est}}$ and thus $\sum_{i = 0}^{N - 1} \expect[\overline{X}_i] \ge N \cdot \epsilon_{\mathrm{est}}$. 
By Chernoff bound, with probability at most $\delta_{\mathrm{est}} / (6|\states||\actions|)$,
\[
\sum_{i = 0}^{N - 1} \overline{X}_i \le 2 N \cdot \epsilon_{\mathrm{est}} / 3.
\]
Hence, by union bound, with probability at least $1 - \delta_{\mathrm{est}} / (3|\states||\actions|)$, 
\[
\sum_{i = 0}^{N - 1} \underline{X}_i < N \cdot \epsilon_{\mathrm{est}} / 3. 
\]
and
\[
\sum_{i = 0}^{N - 1} \overline{X}_i > 2 N \cdot \epsilon_{\mathrm{est}} / 3,
\]
in which case the $\ceil{ N \cdot \epsilon_{\mathrm{est}} / 2}$-th largest element in $F_{s, a}$ is in $\left[\qtile^{\sta}_{\epsilon_{\mathrm{est}}}(s, a), \qtile^{\sta}_{\epsilon_{\mathrm{est}}/4}(s, a)\right]$. 
We finish the proof by a union bound over all $(s, a) \in \states \times \actions$. 
\end{proof}

In Lemma~\ref{lem:approximation}, we show that using the dataset $D$ returned by Algorithm~\ref{alg:collect}, and the estimates of quantiles returned by Algorithm~\ref{alg:estimate}, we can compute accurate estimates of the transition probabilities and rewards. 
The estimators used in Lemma~\ref{lem:approximation} are the empirical estimators, with proper truncation if a list $T_i$ contains too many samples (i.e., more than $\overline{m}^{\sta}(\cdot, \cdot$)). 
As will be made clear in the proof, such truncation is crucial for obtaining estimators with bounded variance. 

\begin{lemma}\label{lem:approximation}
Suppose Algorithm~\ref{alg:estimate} is invoked with the percentile set to be $\epsilon_{\mathrm{est}}$ and the failure probability set to be $\delta$, and Algorithm~\ref{alg:collect} is invoked with $N \ge 16 / \epsilon_{\mathrm{est}} \cdot \log(3|\states||\actions| / \delta)$.
Let $\overline{m}^{\sta} : \states \times \actions \to \mathbb{N}$ be the estimates returned by Algorithm~\ref{alg:estimate}. 
Let $D$ be the dataset returned by Algorithm~\ref{alg:collect}  where
\[
D= \left( \left(\left( s_{i, t}, a_{i, t}, r_{i, t}, s'_{i, t}\right)\right)_{t = 0}^{|\states||\actions| \cdot |\actions|^{2|\states|} \cdot H - 1} \right)_{i = 0}^{N - 1}.
\]
For each $(s, a) \in \states \times \actions$, for each $i \in [N]$ and $t \in \left[|\states||\actions| \cdot |\actions|^{2|\states|} \cdot H\right]$, define
\[
\mathsf{Trunc}_{i, t}(s, a) = \indict\left[ \sum_{t' = 0}^{t - 1}\indict\left[ (s_{i, t'}, a_{i, t'}) = (s, a)\right]  < \overline{m}^{\sta}(s, a) \right].
\]
For each $(s, a, s') \in \states \times \actions \times \states$, define
\[
m_D(s, a) =
 \sum_{i = 0}^{N - 1}
\sum_{t = 0}^{|\states||\actions| \cdot |\actions|^{2|\states|} \cdot H - 1} \indict\left[ (s_{i, t}, a_{i, t}) = (s, a)\right] \cdot\mathsf{Trunc}_{i, t}(s, a) ,
\]
\[
\widehat{P}(s' \mid s, a) =\frac{
\sum_{i = 0}^{N - 1}
\sum_{t = 0}^{|\states||\actions| \cdot |\actions|^{2|\states|} \cdot H - 1}
 \indict\left[ (s_{i, t}, a_{i, t}, s_{i, t}') = (s, a, s')\right] \cdot \mathsf{Trunc}_{i, t}(s, a) 
 }{\max\{1, m_D(s, a)\}},
\]
\[
\widehat{R}(s, a) = \frac{
\sum_{i = 0}^{N - 1}
\sum_{t = 0}^{|\states||\actions| \cdot |\actions|^{2|\states|} \cdot H - 1}
 \indict\left[ (s_{i, t}, a_{i, t}) = (s, a)\right] \cdot r_{i, t} \cdot \mathsf{Trunc}_{i, t}(s, a) 
 }{\max\{1, m_D(s, a)\}}
 \]
and
\[
\widehat{\mu}(s) = \frac{\sum_{i = 0}^{N - 1} \indict[s_0 = s]}{N}.
\]
Then with probability at least $1- \delta$, for all $(s, a, s') \in \states \times \actions \times \states$ with $\qtile^{\sta}_{\epsilon_{\mathrm{est}}}(s, a) > 0$, we have
\begin{align*}
\left| \widehat{P}(s' \mid s, a) - P(s' \mid s, a) \right|  \le &  \max \left \{\frac{512\log(18|\states|^2|\actions| / \delta)}{  \overline{m}^{\sta}(s, a) \cdot N \cdot \epsilon_{\mathrm{est}} } , 32\sqrt{\frac{\widehat{P}(s' \mid s, a)\cdot \log(18|\states|^2|\actions| / \delta)}{\overline{m}^{\sta}(s, a)\cdot N \cdot \epsilon_{\mathrm{est}}} }  \right \} \\ 
\le& \max \left \{ \frac{512\log(18|\states|^2|\actions| / \delta)}{ \qtile^{\sta}_{\epsilon_{\mathrm{est}}}(s, a)\cdot N \cdot \epsilon_{\mathrm{est}} } ,  64\sqrt{\frac{P(s' \mid s, a)\cdot \log(18|\states|^2|\actions| / \delta)}{\qtile^{\sta}_{\epsilon_{\mathrm{est}}}(s, a)\cdot N \cdot \epsilon_{\mathrm{est}}} }\right\},
\end{align*}
\[
\left| \widehat{R}(s' \mid s, a) - \expect[R(s, a)] \right|  \le 
8\sqrt{\frac{ \expect\left[(R(s, a))^2\right]  \cdot \log(18|\states||\actions| / \delta)}{\qtile^{\sta}_{\epsilon_{\mathrm{est}}}(s, a) \cdot N \cdot \epsilon_{\mathrm{est}}}} + \frac{8\log(18|\states||\actions| / \delta)}{\qtile^{\sta}_{\epsilon_{\mathrm{est}}}(s, a) \cdot N \cdot \epsilon_{\mathrm{est}} }, 
\]
and 
\[
\left| \widehat{\mu}(s) -\mu(s) \right| \le \sqrt{ \frac{\log(18|\states| / \delta)}{N}}.
\]
\end{lemma}
\begin{proof}
Fix a state-action pair $(s, a) \in \states \times \actions$ and $s' \in \states$. 
For each $i \in [N]$ and $t \in \left[|\states||\actions| \cdot |\actions|^{2|\states|} \cdot H\right]$, 
let $\mathcal{F}_{i, t}$ be the filtration induced by \[
\left\{ \left(s_{i, t'}, a_{i, t'}, r_{i, t'}, s_{i, t'}'\right)\right\}_{t' = 0}^{t - 1}.
\]
For each $i \in [N]$ and $t \in \left[|\states||\actions| \cdot |\actions|^{2|\states|} \cdot H\right]$, define
\[
X_{i, t}=  \left( 
 \indict\left[ (s_{i, t}, a_{i, t}, s_{i, t}') = (s, a, s')\right]  - P(s' \mid s, a)\indict\left[ (s_{i, t}, a_{i, t}) = (s, a)\right]  \right) 
 \cdot  \mathsf{Trunc}_{i, t}(s, a)
\]
and
\[
Y_{i, t}=  
 \indict\left[ (s_{i, t}, a_{i, t}) = (s, a)\right] \cdot ( r_{i, t} - \expect[R(s, a)] ) \cdot  \mathsf{Trunc}_{i, t}(s, a).
\]
Clearly,
\begin{align*}
&\expect \left[  \indict\left[ (s_{i, t}, a_{i, t}, s_{i, t}') = (s, a, s')\right] \cdot\mathsf{Trunc}_{i, t}(s, a)  \mid \mathcal{F}_{i, t}  \right] \\ = &P(s'  \mid s, a) \cdot \expect \left[  \indict\left[ (s_{i, t}, a_{i, t}) = (s, a)\right] \cdot\mathsf{Trunc}_{i, t}(s, a)\mid  \mathcal{F}_{i, t}  \right]
\end{align*}
and
\begin{align*}
&\expect\left[ \indict\left[ (s_{i, t}, a_{i, t}) = (s, a)\right] \cdot  r_{i, t} \cdot  \mathsf{Trunc}_{i, t}(s, a) \mid \mathcal{F}_{i, t}  \right]\\
=& \expect\left[ \indict\left[ (s_{i, t}, a_{i, t}) = (s, a)\right] \cdot  \expect[R(s, a)] \cdot  \mathsf{Trunc}_{i, t}(s, a) \mid \mathcal{F}_{i, t}  \right],
\end{align*}
which implies
\begin{align*}
&\expect \left[  \indict\left[ (s_{i, t}, a_{i, t}, s_{i, t}') = (s, a, s')\right] \cdot\mathsf{Trunc}_{i, t}(s, a)    \right] \\ = &P(s'  \mid s, a) \cdot \expect \left[  \indict\left[ (s_{i, t}, a_{i, t}) = (s, a)\right] \cdot\mathsf{Trunc}_{i, t}(s, a) \right], 
\end{align*}
and
\begin{align*}
&\expect\left[ \indict\left[ (s_{i, t}, a_{i, t}) = (s, a)\right] \cdot  r_{i, t} \cdot  \mathsf{Trunc}_{i, t}(s, a) \right]\\
=& \expect\left[ \indict\left[ (s_{i, t}, a_{i, t}) = (s, a)\right] \cdot  \expect[R(s, a)] \cdot  \mathsf{Trunc}_{i, t}(s, a) \right],
\end{align*}
and thus 
\[
\expect \left[X_{i, t}   \right] =\expect \left[Y_{i, t}   \right]  = 0.
\]
Moreover, for any $i \in [N]$ and $0 \le t' < t < |\states||\actions| \cdot |\actions|^{2|\states|} \cdot H$, we have
\[
\expect\left[X_{i, t'}  \cdot X_{i, t}   \right] = \expect\left[ \expect\left[X_{i, t'}  \cdot X_{i, t}  \mid \mathcal{F}_{i, t}\right]   \right]  =  \expect\left[ X_{i, t'}  \cdot  \expect\left[X_{i, t}  \mid \mathcal{F}_{i, t}\right]   \right] = 0
\]
and
\[
\expect\left[Y_{i, t'}  \cdot Y_{i, t}   \right] = \expect\left[ \expect\left[Y_{i, t'}  \cdot Y_{i, t}  \mid \mathcal{F}_{i, t}\right]   \right]  =  \expect\left[ Y_{i, t'}  \cdot  \expect\left[Y_{i, t}  \mid \mathcal{F}_{i, t}\right]   \right] = 0.
\]
Note that for each $i \in [N]$, 
\[
\expect \left[  \left( \sum_{t = 0}^{|\states||\actions| \cdot |\actions|^{2|\states|} \cdot H - 1} X_{i, t} \right)^2\right] =  \sum_{t = 0}^{|\states||\actions| \cdot |\actions|^{2|\states|} \cdot H - 1}\expect \left[  \left(X_{i, t}\right)^2\right] 
\]
and
\[
\expect \left[  \left( \sum_{t = 0}^{|\states||\actions| \cdot |\actions|^{2|\states|} \cdot H - 1} Y_{i, t} \right)^2\right] =  \sum_{t = 0}^{|\states||\actions| \cdot |\actions|^{2|\states|} \cdot H - 1}\expect \left[  \left(Y_{i, t}\right)^2\right] .
\]
Furthermore, for each $i \in [N]$ and $t \in  \left[|\states||\actions| \cdot |\actions|^{2|\states|} \cdot H\right]$, 
\begin{align*}
\expect \left[  X_{i, t}^2\right] \le
& \expect\left[\indict\left[ (s_{i, t}, a_{i, t}, s_{i, t}') = (s, a, s')\right] \cdot\mathsf{Trunc}_{i, t}(s, a)  \right] \\ +&  \expect \left[ \left( P(s' \mid s, a)\right)^2 \cdot \indict\left[ (s_{i, t}, a_{i, t}) = (s, a)\right]  \cdot\mathsf{Trunc}_{i, t}(s, a)\right]\\
\le &  2 P(s' \mid s, a) \cdot  \expect \left[  \indict\left[ (s_{i, t}, a_{i, t}) = (s, a)\right]  \cdot\mathsf{Trunc}_{i, t}(s, a)\right]
\end{align*}
and
\begin{align*}
\expect \left[  Y_{i, t}^2\right] \le
& \expect \left[ \indict\left[ (s_{i, t}, a_{i, t}) = (s, a)\right] \cdot ( r_{i, t} - \expect[R(s, a)] )^2 \cdot  \mathsf{Trunc}_{i, t}(s, a) \right]\\
\le & \expect \left[ \indict\left[ (s_{i, t}, a_{i, t}) = (s, a)\right] \cdot  \expect\left[(R(s, a))^2\right]  \cdot  \mathsf{Trunc}_{i, t}(s, a) \right].
\end{align*}
Since
\[
\sum_{t = 0}^{|\states||\actions| \cdot |\actions|^{2|\states|} \cdot H - 1} \indict\left[ (s_{i, t}, a_{i, t}) = (s, a)\right]  \cdot\mathsf{Trunc}_{i, t}(s, a) \le \overline{m}^{\sta}(s, a), 
\]
we have
\[
\expect \left[  \left( \sum_{t = 0}^{|\states||\actions| \cdot |\actions|^{2|\states|} \cdot H - 1} X_{i, t} \right)^2\right]  \le 2P(s' \mid s, a) \cdot \overline{m}^{\sta}(s, a)
\]
and
\[
\expect \left[  \left( \sum_{t = 0}^{|\states||\actions| \cdot |\actions|^{2|\states|} \cdot H - 1} Y_{i, t} \right)^2\right]  \le \expect\left[(R(s, a))^2\right]  \cdot \overline{m}^{\sta}(s, a).
\]

Now, for each $i \in [N]$, define
\[
\mathcal{X}_i = \sum_{t = 0}^{|\states||\actions| \cdot |\actions|^{2|\states|} \cdot H - 1} X_{i, t} 
\]
and
\[
\mathcal{Y}_i = \sum_{t = 0}^{|\states||\actions| \cdot |\actions|^{2|\states|} \cdot H - 1} Y_{i, t} .
\]
We have $\expect[\mathcal{X}_i ] = \expect[\mathcal{Y}_i ] =0$, 
\[\expect[\mathcal{X}_i ^2] \le 2P(s' \mid s, a) \cdot \overline{m}^{\sta}(s, a)\]
and \[\expect[\mathcal{Y}_i^2] \le \expect\left[(R(s, a))^2\right]  \cdot \overline{m}^{\sta}(s, a).\]
Also note that
\[
\sum_{i = 0}^{N - 1}\mathcal{X}_i = \sum_{i = 0}^{N - 1}
\sum_{t = 0}^{|\states||\actions| \cdot |\actions|^{2|\states|} \cdot H - 1}
 \indict\left[ (s_{i, t}, a_{i, t}, s_{i, t}') = (s, a, s')\right] \cdot\mathsf{Trunc}_{i, t}(s, a) 
 -  P(s' \mid s, a) \cdot m_D(s, a)
\]
and
\[
\sum_{i = 0}^{N - 1}\mathcal{Y}_i = \sum_{i = 0}^{N - 1}
\sum_{t = 0}^{|\states||\actions| \cdot |\actions|^{2|\states|} \cdot H - 1}
 \indict\left[ (s_{i, t}, a_{i, t}) = (s, a)\right] \cdot r_{i, t} \cdot \mathsf{Trunc}_{i, t}(s, a) 
 -  \expect[R(s, a)] \cdot m_D(s, a).
\]

By Bernstein's inequality, 
\[
 \Pr \left[ \left|
 \sum_{i = 0}^{N - 1}\mathcal{X}_i
 \right|  \ge t \right] 
 \le  2\exp\left( \frac{-t^2}{2 \cdot  \overline{m}^{\sta}(s, a) \cdot N \cdot P(s' \mid s, a) + t / 3} \right).
\]
Thus, by setting $t = 2\sqrt{\overline{m}^{\sta}(s, a) \cdot N \cdot P(s' \mid s, a)\cdot \log(18|\states|^2|\actions| / \delta)} + \log(18|\states|^2|\actions| / \delta)$, we have
\[
 \Pr \left[ \left|
 \sum_{i = 0}^{N - 1}\mathcal{X}_i
 \right|  \ge t \right]   \le \delta / (9|\states|^2|\actions|).
\]
By applying a union bound over all $(s, a, s') \in \states \times \actions \times \states$, with probability at least $1 - \delta / 9$, for all $(s, a, s') \in \states \times \actions \times \states$,
\[
\left| \widehat{P}(s' \mid s, a) - P(s' \mid s, a) \right|  \le 
\frac{2\sqrt{\overline{m}^{\sta}(s, a) \cdot N \cdot P(s' \mid s, a)\cdot \log(18|\states|^2|\actions| / \delta)}}{m_D(s, a)} + \frac{\log(18|\states|^2|\actions| / \delta)}{m_D(s, a)}, 
\]
which we define to be event $\mathcal{E}_P$. 
Note that conditioned on $\mathcal{E}_P$ and the events in Lemma~\ref{lem:estimate_quantile} and Lemma~\ref{lem:number_exceed_quantile}, we have
\[
	\qtile^{\sta}_{\epsilon_{\mathrm{est}}}(s, a) \le \overline{m}^{\sta}(s, a) \le \qtile^{\sta}_{\epsilon_{\mathrm{est}} / 4}(s, a),
\]
which implies
\[
m_D(s, a) \ge N \cdot \epsilon_{\mathrm{est}} / 8 \cdot  \qtile^{\sta}_{\epsilon_{\mathrm{est}} / 4}(s, a) \ge  N \cdot \epsilon_{\mathrm{est}} / 8 \cdot  \overline{m}^{\sta}(s, a) \ge N \cdot \epsilon_{\mathrm{est}} / 8 \cdot  \qtile^{\sta}_{\epsilon_{\mathrm{est}}}(s, a),
\]
and thus 
\begin{align*}
\left| \widehat{P}(s' \mid s, a) - P(s' \mid s, a) \right|  \le  &
8\sqrt{\frac{P(s' \mid s, a)\cdot \log(18|\states|^2|\actions| / \delta)}{\overline{m}^{\sta}(s, a)\cdot N \cdot \epsilon_{\mathrm{est}}}} + \frac{8\log(18|\states|^2|\actions| / \delta)}{  \overline{m}^{\sta}(s, a) \cdot N \cdot \epsilon_{\mathrm{est}} }\\
\le & 8\sqrt{\frac{P(s' \mid s, a)\cdot \log(18|\states|^2|\actions| / \delta)}{\overline{m}^{\sta}(s, a)\cdot N \cdot \epsilon_{\mathrm{est}}}} + \frac{64\log(18|\states|^2|\actions| / \delta)}{  \overline{m}^{\sta}(s, a) \cdot N \cdot \epsilon_{\mathrm{est}} }\\
\le & \max\left\{ 16\sqrt{\frac{P(s' \mid s, a)\cdot \log(18|\states|^2|\actions| / \delta)}{\overline{m}^{\sta}(s, a)\cdot N \cdot \epsilon_{\mathrm{est}}}} , \frac{512\log(18|\states|^2|\actions| / \delta)}{  \overline{m}^{\sta}(s, a) \cdot N \cdot \epsilon_{\mathrm{est}} } \right\}\\
\end{align*}
for all $(s, a, s') \in \states \times \actions \times \states$.

When 
\[P(s' \mid s, a) \le \frac{1024 \log(18|\states|^2|\actions| / \delta)}{ \overline{m}^{\sta}(s, a) \cdot N \cdot \epsilon_{\mathrm{est}}},\] 
we have
\[
16\sqrt{\frac{P(s' \mid s, a)\cdot \log(18|\states|^2|\actions| / \delta)}{\overline{m}^{\sta}(s, a)\cdot N \cdot \epsilon_{\mathrm{est}}}} \le \frac{512\log(18|\states|^2|\actions| / \delta)}{  \overline{m}^{\sta}(s, a) \cdot N \cdot \epsilon_{\mathrm{est}} } ,
\]
and therefore
\[
\left| \widehat{P}(s' \mid s, a) - P(s' \mid s, a) \right|  \le \frac{512\log(18|\states|^2|\actions| / \delta)}{  \overline{m}^{\sta}(s, a) \cdot N \cdot \epsilon_{\mathrm{est}} } \le \frac{512\log(18|\states|^2|\actions| / \delta)}{ \qtile^{\sta}_{\epsilon_{\mathrm{est}}}(s, a)\cdot N \cdot \epsilon_{\mathrm{est}} }
.
\]
When 
\[P(s' \mid s, a) \ge \frac{1024 \log(18|\states|^2|\actions| / \delta)}{ \overline{m}^{\sta}(s, a) \cdot N \cdot \epsilon_{\mathrm{est}}},\] 
 we have
\[
\left| \widehat{P}(s' \mid s, a) - P(s' \mid s, a) \right|  \le16\sqrt{\frac{P(s' \mid s, a)\cdot \log(18|\states|^2|\actions| / \delta)}{\overline{m}^{\sta}(s, a)\cdot N \cdot \epsilon_{\mathrm{est}}}}  \le P(s' \mid s, a ) / 2
\]
and thus
\[
P(s' \mid s, a)  / 2 \le \widehat{P}(s' \mid s, a) \le 2 P(s' \mid s, a) , 
\]
which implies
\[
\left| \widehat{P}(s' \mid s, a) - P(s' \mid s, a) \right|   \le 32\sqrt{\frac{\widehat{P}(s' \mid s, a)\cdot \log(18|\states|^2|\actions| / \delta)}{\overline{m}^{\sta}(s, a)\cdot N \cdot \epsilon_{\mathrm{est}}} }  \le 64\sqrt{\frac{P(s' \mid s, a)\cdot \log(18|\states|^2|\actions| / \delta)}{\qtile^{\sta}_{\epsilon_{\mathrm{est}}}(s, a)\cdot N \cdot \epsilon_{\mathrm{est}}} } .
\]
Hence, conditioned on $\mathcal{E}_P$ and the events in Lemma~\ref{lem:estimate_quantile} and Lemma~\ref{lem:number_exceed_quantile}, for all $(s, a, s') \in \states \times \actions \times \states$,
\begin{align*}
\left| \widehat{P}(s' \mid s, a) - P(s' \mid s, a) \right|  \le &  \max \left \{\frac{512\log(18|\states|^2|\actions| / \delta)}{  \overline{m}^{\sta}(s, a) \cdot N \cdot \epsilon_{\mathrm{est}} } , 32\sqrt{\frac{\widehat{P}(s' \mid s, a)\cdot \log(18|\states|^2|\actions| / \delta)}{\overline{m}^{\sta}(s, a)\cdot N \cdot \epsilon_{\mathrm{est}}} }  \right \} \\ 
\le& \max \left \{ \frac{512\log(18|\states|^2|\actions| / \delta)}{ \qtile^{\sta}_{\epsilon_{\mathrm{est}}}(s, a)\cdot N \cdot \epsilon_{\mathrm{est}} } ,  64\sqrt{\frac{P(s' \mid s, a)\cdot \log(18|\states|^2|\actions| / \delta)}{\qtile^{\sta}_{\epsilon_{\mathrm{est}}}(s, a)\cdot N \cdot \epsilon_{\mathrm{est}}} }\right\}.
\end{align*}

By Bernstein's inequality, 
\[
 \Pr \left[ \left|
 \sum_{i = 0}^{N - 1}\mathcal{Y}_i
 \right|  \ge t \right] 
 \le  2\exp\left( \frac{-t^2}{\expect\left[(R(s, a))^2\right]  \cdot \overline{m}^{\sta}(s, a) \cdot N+ t / 3} \right).
\]
Thus, by setting $t = 2\sqrt{\expect\left[(R(s, a))^2\right]  \cdot \overline{m}^{\sta}(s, a) \cdot N \cdot \log(18|\states||\actions| / \delta)} + \log(18|\states||\actions| / \delta)$, we have
\[
 \Pr \left[ \left|
 \sum_{i = 0}^{N - 1}\mathcal{Y}_i
 \right|  \ge t \right]   \le \delta / (9|\states||\actions|).
\]
By applying a union bound over all $(s, a, s') \in \states \times \actions$, with probability at least $1 - \delta / 9$, for all $(s, a) \in \states \times \actions$,
\[
\left| \widehat{R}(s' \mid s, a) - \expect[R(s, a)] \right|  \le 
\frac{2\sqrt{ \expect\left[(R(s, a))^2\right]  \cdot \overline{m}^{\sta}(s, a) \cdot N \cdot \log(18|\states||\actions| / \delta)}}{m_D(s, a)} + \frac{\log(18|\states||\actions| / \delta)}{m_D(s, a)}, 
\]
which we define to be event $\mathcal{E}_R$. 
Note that conditioned on $\mathcal{E}_R$ and the events in Lemma~\ref{lem:estimate_quantile} and Lemma~\ref{lem:number_exceed_quantile}, we have
\[
\left| \widehat{R}(s' \mid s, a) - \expect[R(s, a)] \right|  \le 
8\sqrt{\frac{ \expect\left[(R(s, a))^2\right]  \cdot \log(18|\states||\actions| / \delta)}{\qtile^{\sta}_{\epsilon_{\mathrm{est}}}(s, a) \cdot N \cdot \epsilon_{\mathrm{est}}}} + \frac{8\log(18|\states||\actions| / \delta)}{\qtile^{\sta}_{\epsilon_{\mathrm{est}}}(s, a) \cdot N \cdot \epsilon_{\mathrm{est}} }.
\]

%
%

Finally, for each $s \in \states$, for each $i \in [N]$, define 
\[
\mathcal{Z}_{i} = \indict[s_{i, 0} = s]- \mu(s).
\]
Note that 
\[
\sum_{i =0}^{N - 1} \mathcal{Z}_{i} = \sum_{i = 0}^{N - 1} \indict[s_{0, i} = s] - N \cdot \mu(s).
\]
Therefore, by Chernoff bound, with probability at least $1 - \delta / (9|\states|)$ we have
\[
\left| \widehat{\mu}(s) -\mu(s) \right| \le \sqrt{ \frac{\log(18|\states| / \delta)}{K}}.
\]
Hence, with probability at least $1 - \delta / 9$, for all $s \in \states$, we have
\[
\left| \widehat{\mu}(s) -\mu(s) \right| \le \sqrt{ \frac{\log(18|\states| / \delta)}{N}}
\]
which we define to be event $\mathcal{E}_{\mu}$. 

We finish the proof by applying a union bound over $\mathcal{E}_P$, $\mathcal{E}_R$, $\mathcal{E}_{\mu}$ and the events in Lemma~\ref{lem:estimate_quantile} and Lemma~\ref{lem:number_exceed_quantile}. 
\end{proof}

%% file: perturbation.tex
\subsection{Perturbation Analysis}\label{sec:perturbation}
In this section, we establish a perturbation analysis on the value functions which is crucial for the analysis in the next section.
We first recall a few basic facts.
\begin{fact}
\label{fac:log}
Let $|x|\le 1/2$ be a real number, we have
\begin{enumerate}
    \item $x-x^2\le \log(1+x)\le x$;
    \item $1+x\le e^x \le 1+2|x|$.
\end{enumerate}
\end{fact}

We now prove the following lemma using the above facts.

\begin{lemma}
\label{lem:mult-add}
Let $\mb\ge 1$, $\bar{n}\ge n\ge1$ be positive integers.
Let $\epsilon\in [0, 1/(8\bar{n})]$ be some real numbers. 
Let $p\in [1/\mb, 1]^n$ be a vector with $\sum_{i = 1}^n p_i \le 1$.
Let $\delta \in \mathbb{R}^n$ be a vector such that for each $1 \le i \le n$,  $|\delta_i|\le \epsilon \sqrt{p_i/\mb}$ and $\left| \sum_{i = 1}^n \delta_i \right|\le \epsilon \bar{n}/\mb$.
For every $m\in[0,\mb]$ and every $\Gamma\in \mathbb{R}^n$ such that $|\Gamma_i|\in [-\sqrt{p_i\mb}, \sqrt{p_i\mb}]$ for all $1 \le i \le n$, we have
\[
(1- 8\bar{n}\epsilon)\prod_{i=1}^np_i^{p_i m + \Gamma_i} \le \prod_{i=1}^n(p_i + \delta_i)^{p_i m + \Gamma_i} \le (1+ 8\bar{n}\epsilon)\prod_{i=1}^np_i^{p_i m + \Gamma_i}.
\]
\end{lemma}
\begin{proof}
Note that
\[
\prod_{i=1}^n(p_i + \delta_i)^{p_i m + \Gamma_i}
= \prod_{i=1}^n(p_i)^{p_i m + \Gamma_i}
\cdot F
\]
where \[F
=\prod_{i=1}^n\left(1+\frac{\delta_i}{p_i}\right)^{p_i m + \Gamma_i}.
\]
Clearly, 
\[
\log F = 
\sum_{i=1}^n(p_im+\Gamma_i)\log\left(1+\frac{\delta_i}{p_i}\right).
\]
By the choice of $\delta$, we have
\[
\left|\frac{\delta_i}{p_i}\right|\le \epsilon \le \frac{1}{2}.
\]
Using Fact~\ref{fac:log}, for all $1 \le i \le n$, we have
\begin{align*}
\frac{\delta_i}{p_i} - \frac{\delta_i^2}{p_i^2}\le \log\left(1+\frac{\delta_i}{p_i}\right)
\le \frac{\delta_i}{p_i} .
\end{align*}
Hence we,
\begin{align*}
|\log F|
&\le \left|\sum_{i=1}^nm\delta_i\right|+\sum_{i=1}^n \left(\frac{|\Gamma_i| |\delta_i|}{p_i} + \frac{|\Gamma_i| \delta_i^2}{p_i^2}  + \frac{m \delta_i^2}{p_i}\right).
\end{align*}
Note that $\left|\sum_{i=1}^nm\delta_i\right|\le \epsilon  \bar{n}$, $|\Gamma_i||\delta_i| \le \epsilon p_i$, 
$|\Gamma_i|\delta_i^2 \le \epsilon p_i \cdot \epsilon \sqrt{p_i/\mb}
\le \epsilon^2 p_i^2$, and 
$m\delta_i^2\le \epsilon^2 p_i$.
We have,
\begin{align*}
|\log F|\le \epsilon \bar{n} + \epsilon n + \epsilon^2 n + \epsilon^2 n\le 4\bar{n}\epsilon.
\end{align*}
By the choice of $\epsilon$, we have $4\bar{n}\epsilon \le 1/2$,
and therefore
\[
1-8\bar{n}\epsilon \le \exp(\log F) \le 1+8\bar{n}\epsilon.
\]
\end{proof}

In the following lemma, we show that for any $(s, a, s') \in \states \times \actions \times \states$, with probability at least $1 - \delta$, the number of times $(s, a, s')$ is visited can be upper bounded in terms of the $\delta / 2$-quantile of the number of times $(s, a)$ is visited and $P(s' \mid s, a)$. 
\begin{lemma}
\label{lem:upperbound_on_visits}
For a given MDP $\mdp$.
Suppose a random trajectory 
\[
T=((s_0, a_0, r_0), (s_1, a_1, r_1), \ldots, (s_{H-1}, a_{H-1}, r_{H-1}), s_H)
\]
 is obtained by executing a (possibly non-stationary) policy $\pi$ in $\mdp$.
 For any $(s, a, s') \in \states \times \actions \times \states$,
with probability at least 
$1-\delta$, we have
\[
\sum_{h=0}^{H-1}\indict\left[(s,a, s') = (s_h, a_h, s_{h+1})\right]
\le \frac{2\qtile_{\delta/2}\left(\sum_{h=0}^{H-1}\indict\left[(s,a) = (s_h, a_h)\right]\right)+4}{\delta}\cdot P(s'\mid s,a).
\]
\end{lemma}
\begin{proof}
For each $h \in [H]$, define
\[
I_h = \begin{cases}
1 & h = 0\\
\indict\left[
\sum_{t=0}^{h-1}  \indict\left[(s_t,a_t) = (s, a)\right] \le \qtile_{\delta/2}\left(\sum_{h=0}^{H-1}\indict\left[(s_h,a_h) = (s, a)\right]\right)+1\right] & h > 0
\end{cases}.
\]
Let $\mathcal{E}_1$ be the event that for all $h \in [H]$, $I_h =1$.
By definition of $\qtile_{\delta/2}$, 
we have
$
\Pr\left[\mathcal{E}_1\right] \ge 1- \delta/2.
$

For each $h \in [H]$, let $\mathcal{F}_h$ be the filtration induced by $\{(s_0, a_0, r_0), \ldots, (s_{h}, a_{h}, r_{h})\}$.
For each $h \in [H]$, define 
\[X_h = \indict\left[(s,a, s') = (s_h, a_h, s_{h+1})\right] \cdot I_h\]
and
\[Y_h = \indict\left[(s,a) = (s_h, a_h)\right] \cdot I_h.\]

When $h = 0$, we have
\[
\expect[X_h] = \expect[\indict\left[(s,a) = (s_0, a_0)\right] \cdot P(s'\mid s,a) = \expect[Y_h] \cdot P(s'\mid s,a) .
\]
When $h \in [H] \setminus \{0\}$, we have
\begin{align*}
\expect[X_h \mid \mathcal{F}_{h - 1}] =& \expect[\indict\left[(s,a, s') = (s_h, a_h, s_{h+1})\right] \cdot I_h \mid  \mathcal{F}_{h - 1}] \\
=& \expect[\indict\left[(s,a) = (s_h, a_h)\right] \cdot I_h \mid  \mathcal{F}_{h - 1}]\cdot P(s'\mid s,a),
\end{align*}
which implies 
\[
\expect[X_h] =  \expect[\indict\left[(s,a) = (s_h, a_h)\right] \cdot I_h]\cdot P(s'\mid s,a) =  \expect[Y_h] \cdot P(s'\mid s,a) .
\]
Note that \[\sum_{h = 0}^{H - 1}Y_h \le \qtile_{\delta/2}\left(\sum_{h=0}^{H-1}\indict\left[(s,a) = (s_h, a_h)\right]\right) + 2,\] which implies
\[
\expect\left[ \sum_{h =0}^{H - 1} X_h \right] \le \left(  \qtile_{\delta/2}\left(\sum_{h=0}^{H-1}\indict\left[(s,a) = (s_h, a_h)\right]\right) + 2 \right) \cdot P(s'\mid s,a) .
\]

By Markov's inequality, with probability at least $1-\delta/2$,
\[
 \sum_{h=0}^{H-1}  X_h \le  \frac{2\qtile_{\delta/2}\left(\sum_{h=0}^{H-1}\indict\left[(s,a) = (s_h, a_h)\right]\right)+4}{\delta}\cdot P(s'\mid s,a).
\]
which we denote as event $\cE_2$.

Conditioned on $\cE_1 \cap \cE_2$ which happens with probability $1-\delta$, we have
\[
\sum_{h=0}^{H-1}\indict\left[(s,a, s') = (s_h, a_h, s_{h+1})\right] =\sum_{h=0}^{H-1}  X_h \le  \frac{2\qtile_{\delta/2}\left(\sum_{h=0}^{H-1}\indict\left[(s,a) = (s_h, a_h)\right]\right)+4}{\delta}\cdot P(s'\mid s,a).
\]
\end{proof}

In the following lemma, we show that for any $(s, a, s') \in \states \times \actions \times \states$, the number of times $(s, a, s')$ is visited should be close to the number of times $(s, a)$ is visited times $P(s' \mid s, a)$. 
\begin{lemma}
\label{lemma:martingal_concent}
For a given MDP $\mdp$.
Suppose a random trajectory 
\[
T=((s_0, a_0, r_0), (s_1, a_1, r_1), \ldots, (s_{H-1}, a_{H-1}, r_{H-1}), s_H)
\]
 is obtained by executing a policy $\pi$ in $\mdp$.
 For any $(s, a, s') \in \states \times \actions \times \states$,
with probability at least 
$1-\delta$, we have
\begin{align*}
&\left|\sum_{h=0}^{H-1} \indict\left[(s,a, s') = (s_t, a_t, s_{t+1})\right] - P(s'\mid s,a)\cdot \sum_{h=0}^{H-1}\indict\left[(s,a) = (s_t, a_t)\right]\right|\\
\le& \sqrt{\frac{4\qtile_{\delta/2}\left(\sum_{h=0}^{H-1}\indict\left[(s,a) = (s_t, a_t)\right]\right)+8}{\delta}\cdot P(s'\mid s,a)}.
\end{align*}
\end{lemma}

\begin{proof}
For each $h \in [H]$, define
\[
I_h = \begin{cases}
1 & h = 0\\
\indict\left[
\sum_{t=0}^{h-1}  \indict\left[(s,a) = (s_t, a_t)\right] \le \qtile_{\delta/2}\left(\sum_{h=0}^{H-1}\indict\left[(s,a) = (s_h, a_h)\right]\right)+1\right] & h > 0
\end{cases}.
\]
Let $\mathcal{E}_1$ be the event that for all $h \in [H]$, $I_h =1$.
By definition of $\qtile_{\delta/2}$, 
we have
$
\Pr\left[\mathcal{E}_1\right] \ge 1- \delta/2.
$

For each $h \in [H]$, let $\mathcal{F}_h$ be the filtration induced by $\{(s_0, a_0, r_0), \ldots, (s_{h}, a_{h}, r_{h})\}$.
For each $h \in [H]$, define
\[
X_h =  \indict\left[(s,a, s') = (s_h, a_h, s_{h+1})\right]  \cdot I_h - P(s'\mid s,a)  \indict\left[(s,a) = (s_h, a_h)\right]  \cdot I_h.
\]
As we have shown in the proof of Lemma~\ref{lem:upperbound_on_visits}, for each $h \in [H]$, $\expect[X_h] = 0$.
Moreover, for any $0\le h'< h\le H-1$, we have
\[
\expect[X_h X_{h'}]
= 
\expect[\expect[X_hX_{h'}|\cF_{h-1}]]
=\expect[X_{h'}\expect[X_h|\cF_{h-1}]] = 0.
\]
Therefore,
\[
\expect\left[\left(\sum_{h=0}^{H-1}X_h\right)^2\right]
= \expect\left[\sum_{h=0}^{H-1}X_h^2\right].
\]
Note that for each $h \in [H]$.
\begin{align*}
X_h^2 &= (\indict\left[(s,a, s') = (s_h, a_h, s_{h+1})\right]  \cdot I_h - P(s'\mid s,a)  \indict\left[(s,a) = (s_h, a_h)\right]  \cdot I_h)^2 \\
& \le
I_h\cdot \left(\indict\left[(s,a, s') = (s_h, a_h, s_{h+1})\right] + 
\left(P(s'\mid s,a)\right)^2\cdot  \indict\left[(s,a) = (s_h, a_h)\right]  \right).
\end{align*}
As we have shown in the proof of Lemma~\ref{lem:upperbound_on_visits}, for each $h \in [H]$, 
\[
\expect[I_h\cdot \indict\left[(s,a, s') = (s_h, a_h, s_{h+1})\right]]  = \expect[I_h\cdot \indict\left[(s,a) = (s_h, a_h)\right]] \cdot P(s'\mid s,a),
\]
which implies
\begin{align*}
\expect\left[\left(\sum_{h=0}^{H-1}X_h\right)^2\right] &= \expect\left[\sum_{h=0}^{H-1}X_h^2\right]\\
&\le \sum_{h=0}^{H-1} \expect\left[I_h\cdot \left(\indict\left[(s,a, s') = (s_h, a_h, s_{h+1})\right] + 
\left(P(s'\mid s,a)\right)^2\cdot  \indict\left[(s,a) = (s_h, a_h)\right]  \right)\right]\\
&\le 2   P(s'\mid s,a) \cdot \sum_{h=0}^{H-1} \expect\left[ I_h\cdot  \indict\left[(s,a) = (s_h, a_h) \right]\right]\\
&\le 
2  P(s'\mid s,a) \cdot \left(\qtile_{\delta/2}\left(\sum_{h=0}^{H-1}\indict\left[(s,a) = (s_h, a_h)\right]\right) +2\right).
\end{align*}
By Chebyshev's inequality, we have with probability at least $1-\delta/2$,
\[
\left|\sum_{h = 0}^{H - 1}X_h\right| \le 
 \sqrt{\frac{4\qtile_{\delta/2}\left(\sum_{h=0}^{H-1}\indict\left[(s,a) = (s_t, a_t)\right]\right)+8}{\delta}\cdot P(s'\mid s,a)}.
\]
which we denote as event $\cE_2$.

Conditioned on $\cE_1\cap \cE_2$ which happens with probability $1-\delta$, we have
\begin{align*}
&\left|\sum_{h=0}^{H-1} \indict\left[(s,a, s') = (s_t, a_t, s_{t+1})\right] - P(s'\mid s,a)\cdot \sum_{h=0}^{H-1}\indict\left[(s,a) = (s_t, a_t)\right]\right|\\
\le& \sqrt{\frac{4\qtile_{\delta/2}\left(\sum_{h=0}^{H-1}\indict\left[(s,a) = (s_t, a_t)\right]\right)+8}{\delta}\cdot P(s'\mid s,a)}.
\end{align*}
\end{proof}

Using Lemma~\ref{lem:mult-add}, Lemma~\ref{lem:upperbound_on_visits} and Lemma~\ref{lemma:martingal_concent}, we now present the main result in this section, which shows that for two MDPs $M$ and $\widehat{M}$ that are close enough in terms of rewards and transition probabilities, for any policy $\pi$, its value in $\widehat{M}$ should be lower bounded by that in $M$ up to an error of $\epsilon$. 

\begin{lemma}\label{lem:perturbation_lower_bound}
Let $M=(\states,\actions, P, R, H, \mu)$ be an MDP and $\pi$ be a policy.
Let $0 < \epsilon \le 1/2$ be a parameter.
For each $(s,a)\in \states\times\actions$, define $\mb(s,a):=\qtile^{\pi}_{\epsilon/(12|\states||\actions|)}(s,a)$.
Let $\widehat{M}=(\states,\actions, \widehat{P}, \widehat{R}, H, \widehat{\mu})$ be another MDP.
If for all $(s,a, s') \in \states \times \actions \times \states$ with $\mb(s,a)\ge 1$, we have
 \[|\widehat{P}(s'|s,a) - P(s'|s,a)|\le \frac{\epsilon}{96|\cS|^2|\cA|}\cdot \max\left( \sqrt{\frac{\epsilon{P}(s'|s,a)}{72\cdot \mb(s,a)\cdot |\states||\actions|}}, \frac{\epsilon}{72\cdot \mb(s,a)\cdot |\states||\actions|}\right),\]
 \[
    \left|\expect[\widehat{R}|s,a] - 
    \expect[{R}|s,a] \right| \le \frac{\epsilon}{24|\cS||\cA|}\cdot \max \left \{ \sqrt{\frac{\expect[(R(s, a))^2]}{\mb(s,a)}} , \frac{1}{\mb(s, a)}\right \}
    \]
    and
    \[
    \left|\mu(s) - \widehat{\mu}(s)\right| \le \epsilon / (6|\states|), 
    \]
then
\[
V^{\pi}_{\widehat{M},H}\ge 
V^{\pi}_{{M}, H} - \epsilon.
\]
\end{lemma}
\begin{proof}
Define $\mathcal{T} = (\mathcal{S} \times \mathcal{A})^H \times \mathcal{S}$ be set of all possible trajectories, where for each $T \in \mathcal{T}$, $T$ has the form
\[
((s_0, a_0), (s_1, a_1), \ldots, (s_{H - 1}, a_{H - 1}), s_H). 
\]
For a trajectory $T\in \cT$, for each $(s, a, s') \in \mathcal{S} \times \mathcal{A} \times \mathcal{S}$, we write
\[m_T(s,a)=\sum_{h=0}^{H - 1} \indict[(s_h,a_h)=(s,a)]\]
as the number times $(s,a)$ is visited and 
\[
m_T(s,a,s')=\sum_{h= 0}^{H - 1}\indict[(s_h,a_h, s_{h+1})=(s,a, s')]
\]
as the number of times $(s,a,s')$ is visited. 
We say a trajectory \[T = ((s_0, a_0), (s_1, a_1), \ldots, (s_{H - 1}, a_{H - 1}), s_H) \in \mathcal{T}\]is {\em compatible} with a (possibly non-stationary) policy $\pi$ if for all $h \in [H]$,
\[
a_h = \pi_h(s_h).
\]
For a (possibly non-stationary) policy $\pi$, we use $\mathcal{T}^{\pi} \subseteq \mathcal{T}$ to denote the set of all trajectories that are compatible with $\pi$.

For an MDP $M=(\states,\actions, P, R, H, \mu)$ and a (possibly non-stationary) policy $\pi$, for a trajectory $T$ that is compatible with $\pi$, we write
\[
p(T, M, \pi) = \mu(s_0) \cdot \prod_{h = 0}^{H - 1} P(s_{h + 1} \mid s_h, a_h) = \mu(s_0) \cdot \prod_{(s, a, s') \in \mathcal{S} \times \mathcal{A} \times \mathcal{S}} P(s' \mid s, a)^{m_T(s, a, s')} 
\]
to be the probability of $T$ when executing $\pi$ in $M$. Here we assume $0^0 = 1$.

Using these definitions, we have
\[
V^{\pi}_{M, H} = \sum_{T \in \mathcal{T}^{\pi}} p(T, M, \pi) \cdot \left( \sum_{(s, a) \in \states \times \actions} m_T(s, a) \cdot \expect[R(s, a)]\right).
\]

Note that for any trajectory $T = ((s_0, a_0), (s_1, a_1), \ldots, (s_{H - 1}, a_{H - 1}), s_H) \in \mathcal{T}^{\pi}$, if $p(T, M, \pi) > 0$, by Assumption~\ref{assump:reward}, 
\[
\sum_{(s, a) \in \states \times \actions} m_T(s, a) \cdot \expect[R(s, a)] \le 1.
\]


We define $\cT^{\pi}_1 \subseteq \cT^{\pi}$ be the set of trajectories that 
for each $T \in \cT^{\pi}_1$, 
for each $(s,a) \in \states \times \actions$, 
\[
m_T(s,a)\le \mb(s,a).
\]
By a union bound over all $(s, a) \in \states \times \actions$, we have
\[
\sum_{T \in \mathcal{T}^{\pi} \setminus  \mathcal{T}^{\pi}_1} p(T, M, \pi) \le \epsilon / 6. 
\]

We also define $\cT^{\pi}_2 \subseteq \cT^{\pi}$ be the set of trajectories that for each $T \in  \cT^{\pi}_2$, for each $(s, a, s') \in \states \times \actions \times \states$,
\[
\left| m_T(s,a,s') - m_T(s,a) \cdot P(s'|s,a) \right| \le \sqrt{\frac{6P(s'|s,a)(4\mb(s,a)+8)|\states||\actions|}{\epsilon}}.
\]
By Lemma~\ref{lemma:martingal_concent} and a union bound over all $(s, a) \in \states\times \actions$, we have 
\[
\sum_{T \in \mathcal{T}^{\pi} \setminus  \mathcal{T}^{\pi}_2} p(T, M, \pi) \le \epsilon / 6. 
\]

Finally, we define $\cT^{\pi}_3 \subseteq \cT^{\pi}$ be the set of trajectories such that for each $T \in  \cT^{\pi}_3$, for each $(s, a, s') \in \states \times \actions \times \states$,
\[
m_T(s,a, s')\le \frac{6|\states||\actions|(2\mb(s,a)+4)}{\epsilon}\cdot P(s'|s,a).
\]
By Lemma~\ref{lem:upperbound_on_visits} and a union bound over all $(s, a) \in \states \times \actions$, we have
\[
\sum_{T \in \mathcal{T}^{\pi} \setminus  \mathcal{T}^{\pi}_3} p(T, M, \pi) \le \epsilon / 6. 
\]

Thus, by defining $\mathcal{T}^{\pi}_{\mathrm{pruned}} =  \mathcal{T}^{\pi}_1  \cap  \mathcal{T}^{\pi}_2 \cap  \mathcal{T}^{\pi}_3$, we have
\[
 \sum_{T \in \mathcal{T}^{\pi}_{\mathrm{pruned}} } p(T, M, \pi) \cdot \left( \sum_{(s, a) \in \states \times \actions} m_T(s, a) \cdot \expect[R(s, a)]\right) \ge V^{\pi}_{M, H}  - \epsilon / 2.
\]

%
%
%
%
Note that for each $T \in \mathcal{T}^{\pi}_{\mathrm{pruned}}$ with
\[
T = ((s_0, a_0), (s_1, a_1), \ldots, (s_{H - 1}, a_{H - 1}), s_H), 
\]
for each state-action $(s,a) \in \states \times \actions$ , we must have $m_T(s, a)  \le \mb(s,a)$.
This is because $T \in \mathcal{T}^{\pi}_1$
Moreover, for any $(s, a, s') \in \states \times \actions \times \states$, if $m_T(s, a, s') \ge 1$, then
\[
P(s' \mid s,a)\ge \frac{\epsilon}{36|\cS||\cA|\mb(s,a)}.
\]
This is because $T \in \mathcal{T}^{\pi}_3$, and if \[P(s' \mid s,a)< \frac{\epsilon}{36|\cS||\cA|\mb(s,a)},\] then 
\[
m_T(s,a, s') 
\le \frac{6|\states||\actions|(2\mb(s,a)+4)}{\epsilon}\cdot P(s'|s,a)
\le  \frac{36|\states||\actions|\mb(s,a)}{\epsilon}\cdot P(s'|s,a) <1.
\]
For each $(s, a) \in \states \times \actions$, define 
\[
\states_{s,a}=
\left\{
s'\in \states \
\mid P(s'|s,a)\ge \frac{\epsilon}{36|\cS||\cA|\mb(s,a)}
\right\}. 
\]
Therefore, for each $(s, a) \in \states \times \actions$, 
\[
\prod_{s' \in \mathcal{S}} P(s' \mid s, a)^{m_T(s, a, s')}  = \prod_{s'\in \states_{s,a}}P(s' \mid s, a)^{m_T(s, a, s')}
\]
and
\[
\prod_{s' \in \mathcal{S}} \widehat{P}(s' \mid s, a)^{m_T(s, a, s')}  = \prod_{s'\in \states_{s,a}} \widehat{P}(s' \mid s, a)^{m_T(s, a, s')}
\]
with  
\begin{align*}
\left| m_T(s,a,s') - m_T(s,a) \cdot P(s'|s,a) \right|& \le \sqrt{\frac{6P(s'|s,a)(4\mb(s,a)+8)|\states||\actions|}{\epsilon}} \\
&\le  \sqrt{\frac{72P(s'|s,a)\mb(s,a)|\states||\actions|}{\epsilon}}.
\end{align*}
Note that $\sum_{s'\in \states}\widehat{P}(s'|s,a) - {P}(s'|s,a)
=0$, which implies 
\[
\left|\sum_{s'\in \states_{s,a}}\widehat{P}(s'|s,a) - {P}(s'|s,a)\right|
=
\left|\sum_{s'\not\in \states_{s,a}}{P}(s'|s,a) - \widehat{P}(s'|s,a)\right|
\le \frac{\epsilon}{96|\states||\actions|}\cdot \frac{\epsilon}{72 \cdot \mb{(s,a)}\cdot |\states||\actions|}.
\]

By applying Lemma~\ref{lem:mult-add}
and settting $\bar{n}$ to be $|\states|$, $n$ to be $|\states_{s,a}|$, $\epsilon$ to be $ \epsilon/(96|\states|^2|\actions|)$, and $\mb$ to be $72 \cdot \mb(s,a)\cdot |\states||\actions| / \epsilon$,
we have
\[
 \prod_{s'\in \states_{s,a}} \widehat{P}(s' \mid s, a)^{m_T(s, a, s')} 
\ge \left(1 - \frac{\epsilon}{12|\states||\actions|}\right)  \prod_{s'\in \states_{s,a}} P(s' \mid s, a)^{m_T(s, a, s')} .
\]
Therefore,
\[
\prod_{(s, a) \in \states \times \actions}  \prod_{s'\in \states} \widehat{P}(s' \mid s, a)^{m_T(s, a, s')} 
\ge \left(1 - \frac{\epsilon}{12|\states||\actions|}\right)^{|\states||\actions|} \prod_{(s, a) \in \states \times \actions}  \prod_{s'\in \states} P(s' \mid s, a)^{m_T(s, a, s')},
\]
which implies
\[
\prod_{(s, a) \in \states \times \actions}  \prod_{s'\in \states} \widehat{P}(s' \mid s, a)^{m_T(s, a, s')} 
\ge (1 - \epsilon / 6) \prod_{(s, a) \in \states \times \actions}  \prod_{s'\in \states} P(s' \mid s, a)^{m_T(s, a, s')}.
\]
For the summation of rewards, we have
\begin{align*}
&\sum_{(s,a)\in \states \times \actions} m_T(s, a) \cdot \left| \expect\left[R(s, a)\right] - \expect\left[\widehat{R}(s, a)\right] \right|\\
= &\sum_{(s, a) \in \states \times \actions \mid m_T(s,a)=1} m_T(s, a) \cdot \left| \expect\left[R(s, a)\right] - \expect\left[\widehat{R}(s, a)\right] \right|\\
+ &\sum_{(s, a) \in \states \times \actions \mid m_T(s,a)>1} m_T(s, a) \cdot \left| \expect\left[R(s, a)\right] - \expect\left[\widehat{R}(s, a)\right] \right|.
\end{align*}
For those $(s, a) \in \states \times \actions$ with $m_T(s,a)>1$, we have $\mb(s, a) > 1$.
By Lemma~\ref{lem:reward}, we have
\begin{align*}
&|\expect[\widehat{R}(s, a)] - \expect[R(s, a)]|\\
\le&  \frac{\epsilon}{24|\cS||\cA|}\cdot \max\left\{ \sqrt{\frac{\expect[(R(s, a))^2]}{\mb(s,a)}}, \frac{1}{\mb(s, a)} \right\}\le 
\max\left\{\frac{\epsilon}{12H}, \frac{\epsilon}{24 |\states| |\actions| \mb(s, a)}\right\}.
\end{align*}
Since $\sum_{(s, a)\in\states\times\actions} m_T(s,a)\le H$ and $m_T(s, a) \le \mb(s, a)$, we have
\[
\sum_{(s, a) \in \states \times \actions \mid m_T(s,a)>1} m_T(s, a) \cdot \left| \expect\left[R(s, a)\right] - \expect\left[\widehat{R}(s, a)\right] \right|
\le 
\frac{\epsilon}{8}.
\]
For those $(s, a) \in \states \times \actions$ with $m_T(s,a)= 1$, we have
\[
\sum_{(s, a) \in \states \times \actions \mid m_T(s,a)=1} m_T(s, a) \cdot \left| \expect\left[R(s, a)\right] - \expect\left[\widehat{R}(s, a)\right] \right| \le \frac{\epsilon}{24}.
\]
Thus,
\[
\sum_{(s,a)\in \states \times \actions} m_T(s, a) \cdot \left| \expect\left[R(s, a)\right] - \expect\left[\widehat{R}(s, a)\right] \right| \le \frac{\epsilon}{6}.
\]

For each $T \in \mathcal{T}^{\pi}_{\mathrm{pruned}}$ with
\[
T = ((s_0, a_0), (s_1, a_1), \ldots, (s_{H - 1}, a_{H - 1}), s_H), 
\]
we have
\begin{align*}
& p(T, \widehat{M}, \pi) \cdot \left( \sum_{(s, a) \in \states \times \actions} m_T(s, a) \cdot \expect[\widehat{R}(s, a)]\right) \\
= & \widehat{\mu}(s_0)  \cdot \prod_{(s, a) \in \states \times \actions} \prod_{s' \in \states} \widehat{P}(s' \mid s, a)^{m_T(s, a, s')} \cdot \left( \sum_{(s, a) \in \states \times \actions} m_T(s, a) \cdot \expect[\widehat{R}(s, a)]\right)\\
\ge & (\mu(s_0)  - \epsilon / (6|\states|)) \cdot  (1 - \epsilon / 6) \prod_{(s, a) \in \states \times \actions} \prod_{s' \in \states} P(s' \mid s, a)^{m_T(s, a, s')}  \cdot  \left( \sum_{(s, a) \in \states \times \actions} m_T(s, a) \cdot \expect[R(s, a)] - \epsilon / 6\right).
\end{align*}

Since 
\[
 \sum_{T \in \mathcal{T}^{\pi}_{\mathrm{pruned}} } p(T, M, \pi) \cdot \left( \sum_{(s, a) \in \states \times \actions} m_T(s, a) \cdot \expect[R(s, a)]\right) \ge V^{\pi}_{M, H}  - \epsilon / 2,
\]
we have
\[
V^{\pi}_{\widehat{M}, H} \ge  \sum_{T \in \mathcal{T}^{\pi}_{\mathrm{pruned}} } p(T, \widehat{M}, \pi) \cdot \left( \sum_{(s, a) \in \states \times \actions} m_T(s, a) \cdot \expect[\widehat{R}(s, a)]\right) \ge V^{\pi}_{M, H}  - \epsilon.
\]

\end{proof}

Now we show that for two MDPs $M$ and $\widehat{M}$ with the same transition probabilities and close enough rewards, for any policy $\pi$, its value in $\widehat{M}$ should be upper bounded by that in $\widehat{M}$ up to an error of $\epsilon$. 

\begin{lemma}\label{lem:perturbation_upper_bound}
Let $M=(\states,\actions, P, R, H, \mu)$ be an MDP and $\pi$ be a policy.
Let $0 < \epsilon \le 1/2$ be a parameter.
For each $(s,a)\in \states\times\actions$, define $\mb(s,a)=\qtile^{\pi}_{\epsilon/(12|\states||\actions|)}(s,a)$.
Let $\widehat{M}=(\states,\actions, P, \widehat{R}, H, \mu)$ be another MDP.
If for all $(s,a) \in \states \times \actions$ with $\mb(s,a)\ge 1$, we have
    \[
    \left|\expect[\widehat{R}|s,a] - 
    \expect[{R}|s,a] \right| \le \frac{\epsilon}{24|\cS||\cA|}\cdot \max \left \{ \sqrt{\frac{\expect[(R(s, a))^2]}{\mb(s,a)}} , \frac{1}{\mb(s, a)}\right \}.
    \]
then
\[
V^{\pi}_{\widehat{M},H}\le 
V^{\pi}_{{M}, H} + \epsilon.
\]
\end{lemma}
\begin{proof}
We adopt the same notations as in the proof of Lemma~\ref{lem:perturbation_lower_bound}.
Recall that 
\[
V^{\pi}_{M, H} = \sum_{T \in \mathcal{T}^{\pi}} p(T, M, \pi) \cdot \left( \sum_{(s, a) \in \states \times \actions} m_T(s, a) \cdot \expect[R(s, a)]\right).
\]
and
\[
\sum_{T \in \mathcal{T}^{\pi} \setminus \mathcal{T}^{\pi}_1} p(T, M, \pi) \le \epsilon / 6. 
\]

As in the proof of Lemma~\ref{lem:perturbation_lower_bound}, for the summation of rewards, we have
\begin{align*}
&\sum_{(s,a)\in \states \times \actions} m_T(s, a) \cdot \left| \expect\left[R(s, a)\right] - \expect\left[\widehat{R}(s, a)\right] \right|\\
= &\sum_{(s, a) \in \states \times \actions \mid m_T(s,a)=1} m_T(s, a) \cdot \left| \expect\left[R(s, a)\right] - \expect\left[\widehat{R}(s, a)\right] \right|\\
+ &\sum_{(s, a) \in \states \times \actions \mid m_T(s,a)>1} m_T(s, a) \cdot \left| \expect\left[R(s, a)\right] - \expect\left[\widehat{R}(s, a)\right] \right|.
\end{align*}

For each $T \in \mathcal{T}^{\pi}_1$, for each $(s, a) \in \states \times \actions$, we must have $m_T(s, a) \le \mb(s, a)$.
Therefore, for those $(s, a) \in \states \times \actions$ with $m_T(s, a) > 1$, by Lemma~\ref{lem:reward}, we have
\begin{align*}
&|\expect[\widehat{R}(s, a)] - \expect[R(s, a)]|\\
\le&  \frac{\epsilon}{24|\cS||\cA|}\cdot \max\left\{ \sqrt{\frac{\expect[(R(s, a))^2]}{\mb(s,a)}}, \frac{1}{\mb(s, a)} \right\}\le 
\max\left\{\frac{\epsilon}{12H}, \frac{\epsilon}{24 |\states| |\actions| \mb(s, a)}\right\}.
\end{align*}
Since $\sum_{(s, a)\in\states\times\actions} m_T(s,a)\le H$ and $m_T(s, a) \le \mb(s, a)$, we have
\[
\sum_{(s, a) \in \states \times \actions \mid m_T(s,a)>1} m_T(s, a) \cdot \left| \expect\left[R(s, a)\right] - \expect\left[\widehat{R}(s, a)\right] \right|
\le 
\frac{\epsilon}{8}.
\]
For those $(s, a) \in \states \times \actions$ with $m_T(s,a)= 1$, we have
\[
\sum_{(s, a) \in \states \times \actions \mid m_T(s,a)=1} m_T(s, a) \cdot \left| \expect\left[R(s, a)\right] - \expect\left[\widehat{R}(s, a)\right] \right| \le \frac{\epsilon}{24}.
\]
Thus,
\[
\sum_{(s,a)\in \states \times \actions} m_T(s, a) \cdot \left| \expect\left[R(s, a)\right] - \expect\left[\widehat{R}(s, a)\right] \right| \le \frac{\epsilon}{6}.
\]
Hence,
\begin{align*}
V^{\pi}_{\widehat{M},\mu}&\le \sum_{T \in \mathcal{T}^{\pi} \setminus \mathcal{T}^{\pi}_1} p(T, M, \pi) + \sum_{T \in \mathcal{T}^{\pi}_1} p(T, M, \pi) \cdot \left( \sum_{(s, a) \in \states \times \actions} m_T(s, a) \cdot \expect[\widehat{R}(s, a)]\right)\\
& \le \epsilon / 6 + \sum_{T \in \mathcal{T}^{\pi}_1} p(T, M, \pi) \cdot \left( \sum_{(s, a) \in \states \times \actions} m_T(s, a) \cdot \expect[R(s, a)] + \epsilon / 6\right)\\
& \le V^{\pi}_{M,\mu} + \epsilon.
\end{align*}

\end{proof}

%% file: planning.tex
\subsection{Pessimistic Planning}\label{sec:planning}

We now present our final algorithm in the RL setting.
The formal description is provided in Algorithm~\ref{alg:planning}.
In our algorithm, we first invoke Algorithm~\ref{alg:collect} to collect a dataset $D$ and then invoke Algorithm~\ref{alg:estimate} to estimate $\qtile^{\sta}_{\epsilon}(s, a)$ for some properly chosen $\epsilon$. 
We then use the estimators in Lemma~\ref{lem:approximation} to define $\widehat{P}$, $\widehat{R}$ and $\widehat{\mu}$.
Note that Lemma~\ref{lem:approximation} not only provides an estimator but also provides a computable confidence interval for $\widehat{P}$ and $\widehat{\mu}$, which we also utilize in our algorithm.

At this point, a natural idea is to find the optimal policy with respect to the MDP $\widehat{M}$ defined by $\widehat{P}$, $\widehat{R}$ and $\widehat{\mu}$.
However, our Lemma~\ref{lem:perturbation_lower_bound} only provides a lower bound guarantee for $V^{\pi}_{\widehat{M}, H}$ without any upper bound guarantee. 
We resolve this issue by pessimistic planning. More specifically, for any policy $\pi$, we define its pessimistic value to be 
\[
\underline{V}^{\pi} = \min_{M \in \mathcal{M}}V^{\pi}_{M, H}
\]
where $\mathcal{M}$ includes all MDPs whose transition probabilities are within the confidence interval provided in Lemma~\ref{lem:approximation}. 
We simply return the policy $\pi$ that maximizes $\underline{V}^{\pi}$. 
Since the true MDP lies in $\mathcal{M}$, $\underline{V}^{\pi}$ is never an overestimate. 
On the other hand, Lemma~\ref{lem:perturbation_lower_bound} guarantees that $\underline{V}^{\pi}$ is also lower bounded by $V^{\pi}_{M, H}$ up to an error of $\epsilon$.
Therefore, $\underline{V}^{\pi}$ provides an accurate estimate to the true value of $\pi$.
However, note that Lemma~\ref{lem:approximation} does not provide a computable confidence interval for the rewards.
Fortunately, as we have shown in Lemma~\ref{lem:perturbation_upper_bound}, perturbation on the rewards will not significantly increase the value of the policy and thus the estimate is still accurate.

\begin{algorithm}
  \caption{Pessimistic Planning}\label{alg:planning}
  \begin{algorithmic}[1]
  \State \textbf{Input:} desired accuracy $\epsilon$, failure probability $\delta$
  \State \textbf{Output:} An $\epsilon$-optimal policy $\pi$

  \State Invoke Algorithm~\ref{alg:collect}  with \[N = 2^{66} \cdot (|\states| + 1)^{24(|\states| + 1)} \cdot \log(18 |\states|^2|\actions| / \delta) \cdot |\states|^7 |\actions|^5 / \varepsilon^5\]
  and receive
\[
D= \left( \left(\left( s_{i, t}, a_{i, t}, r_{i, t}, s'_{i, t}\right)\right)_{t = 0}^{|\states||\actions| \cdot |\actions|^{2|\states|} \cdot H - 1} \right)_{i = 0}^{N - 1}
\]
  \State Invoke Algorithm~\ref{alg:estimate} with\[\epsilon_{\mathrm{est}} = \frac{\epsilon}{32768 \cdot |\states||\actions| (|\states| + 1)^{12(|\states| + 1)}}\] and $\delta = \delta$, and receive estimates $\overline{m}^{\sta} : \states \times \actions \to \mathbb{N}$

\State For each $(s, a) \in \states \times \actions$, for each $i \in [N]$ and $t \in \left[|\states||\actions| \cdot |\actions|^{2|\states|} \cdot H\right]$, define
\[
\mathsf{Trunc}_{i, t}(s, a) = \indict\left[ \sum_{t' = 0}^{t - 1}\indict\left[ (s_{i, t'}, a_{i, t'}) = (s, a)\right]  < \overline{m}^{\sta}(s, a) \right]
\]
\State For each $(s, a, s') \in \states \times \actions \times \states$, define
\[
m_D(s, a) =
 \sum_{i = 0}^{N - 1}
\sum_{t = 0}^{|\states||\actions| \cdot |\actions|^{2|\states|} \cdot H - 1} \indict\left[ (s_{i, t}, a_{i, t}) = (s, a)\right] \cdot\mathsf{Trunc}_{i, t}(s, a) ,
\]
\[
\widehat{P}(s' \mid s, a) =\frac{
\sum_{i = 0}^{N - 1}
\sum_{t = 0}^{|\states||\actions| \cdot |\actions|^{2|\states|} \cdot H - 1}
 \indict\left[ (s_{i, t}, a_{i, t}, s_{i, t}') = (s, a, s')\right] \cdot \mathsf{Trunc}_{i, t}(s, a) 
 }{\max\{1, m_D(s, a)\}},
\]
\[
\widehat{R}(s, a) = \frac{
\sum_{i = 0}^{N - 1}
\sum_{t = 0}^{|\states||\actions| \cdot |\actions|^{2|\states|} \cdot H - 1}
 \indict\left[ (s_{i, t}, a_{i, t}) = (s, a)\right] \cdot r_{i, t} \cdot \mathsf{Trunc}_{i, t}(s, a) 
 }{\max\{1, m_D(s, a)\}}
 \]
and
\[
\widehat{\mu}(s) = \frac{\sum_{i = 0}^{N - 1} \indict[s_0 = s]}{N}
\]

  \State Define $\mathcal{M}$ to be a set of MDPs where for each $M = (\states, \actions, \widetilde{P}, \widehat{R}, H, \widetilde{\mu}) \in \mathcal{M}$,  
  \[
\left| \widehat{P}(s' \mid s, a) - \widetilde{P}(s' \mid s, a) \right| 
\le \max \left \{\frac{512\log(18|\states|^2|\actions| / \delta)}{  \overline{m}^{\sta}(s, a) \cdot N \cdot \epsilon_{\mathrm{est}} } , 32\sqrt{\frac{\widehat{P}(s' \mid s, a)\cdot \log(18|\states|^2|\actions| / \delta)}{\overline{m}^{\sta}(s, a)\cdot N \cdot \epsilon_{\mathrm{est}}} }  \right \}
\]
and
\[
\left| \widehat{\mu}(s) -\widetilde{\mu}(s) \right| \le \sqrt{ \frac{\log(18|\states| / \delta)}{N}}
\]
  for all $(s, a, s') \in \states \times \actions \times \states$
\State For each (possibly non-stationary) policy $\pi$, define 
$\underline{V}^{\pi} = \min_{M \in \mathcal{M}}V^{\pi}_{M, H}$ \label{alg:define_value}
\State \Return{$\mathrm{armgax}_{\pi} \underline{V}^{\pi}$}

  \end{algorithmic}
\end{algorithm}

We now present the formal analysis for our algorithm.

\begin{theorem}\label{thm:rl_oracle}
With probability at least $1 - \delta$, for any (possibly non-stationary) policy $\pi$, 
\[
\left | \underline{V}^{\pi} - V^{\pi}_{M, H} \right| \le   \epsilon / 2
\]
where $\underline{V}^{\pi}$ is defined in Line~\ref{alg:define_value} of Algorithm~\ref{alg:planning}.
Moreover, Alogrithm~\ref{alg:planning} samples at most 
\[
2^{66} \cdot (|\states| + 1)^{24(|\states| + 1)} \cdot |\actions|^{2|\states|} \cdot \log(12 |\states|^2|\actions| / \delta) \cdot |\states|^8 |\actions|^6 / \varepsilon^5.
\]
trajectories. 
\end{theorem}
\begin{proof}
%
By Lemma~\ref{lem:approximation}, with probability at least $1- \delta$, for all $(s, a, s') \in \states \times \actions \times \states$, we have
\begin{align*}
\left| \widehat{P}(s' \mid s, a) - P(s' \mid s, a) \right|  \le &  \max \left \{\frac{512\log(18|\states|^2|\actions| / \delta)}{  \overline{m}^{\sta}(s, a) \cdot N \cdot \epsilon_{\mathrm{est}} } , 32\sqrt{\frac{\widehat{P}(s' \mid s, a)\cdot \log(18|\states|^2|\actions| / \delta)}{\overline{m}^{\sta}(s, a)\cdot N \cdot \epsilon_{\mathrm{est}}} }  \right \} \\ 
\le& \max \left \{ \frac{512\log(18|\states|^2|\actions| / \delta)}{ \qtile^{\sta}_{\epsilon_{\mathrm{est}}}(s, a)\cdot N \cdot \epsilon_{\mathrm{est}} } ,  64\sqrt{\frac{P(s' \mid s, a)\cdot \log(18|\states|^2|\actions| / \delta)}{\qtile^{\sta}_{\epsilon_{\mathrm{est}}}(s, a)\cdot N \cdot \epsilon_{\mathrm{est}}} }\right\},
\end{align*}
\[
\left| \widehat{R}(s' \mid s, a) - \expect[R(s, a)] \right|  \le 
8\sqrt{\frac{ \expect\left[(R(s, a))^2\right]  \cdot \log(18|\states||\actions| / \delta)}{\qtile^{\sta}_{\epsilon_{\mathrm{est}}}(s, a) \cdot N \cdot \epsilon_{\mathrm{est}}}} + \frac{8\log(18|\states||\actions| / \delta)}{\qtile^{\sta}_{\epsilon_{\mathrm{est}}}(s, a) \cdot N \cdot \epsilon_{\mathrm{est}} }, 
\]
and 
\[
\left| \widehat{\mu}(s) -\mu(s) \right| \le \sqrt{ \frac{\log(18|\states| / \delta)}{N}}.
\]
In the remaining part of the analysis, we condition on the above event. 

by Lemma~\ref{lem:stationary_quantile}, for any (possibly non-stationary) policy $\pi$,
\[
\qtile^{\sta}_{\epsilon_{\mathrm{est}}} (s, a)\ge  \frac{1}{4096 \cdot |\states|^{12|\states|}} \cdot \epsilon / (24 |\states||\actions|)\cdot \qtile^{\pi}_{ \epsilon / (24 |\states||\actions|)}(s, a) .
\]
Let $M = (\states, \actions, P, R, H, \mu)$ be the true MDP. 
By Lemma~\ref{lem:perturbation_lower_bound}, for any (possibly non-stationary) policy $\pi$, for any $\overline{M} \in \mathcal{M}$, we have
\[
V^{\pi}_{\overline{M}, H} \ge V^{\pi}_{M, H} - \epsilon / 2.  
\]
Moreover, $M' = (\states, \actions, P, \widehat{R}, H, \mu) \in \mathcal{M}$.
Therefore, by Lemma~\ref{lem:perturbation_upper_bound},
\[
V^{\pi}_{M', H} \le V^{\pi}_{M, H} + \epsilon / 2.
\]
Consequently,
\[
\left | \underline{V}^{\pi} - V^{\pi}_{M, H} \right| \le   \epsilon / 2.
\]

Finally, the algorithm samples at most 
\begin{align*}
&   (N  +  \ceil{300\log(6|\states||\actions| / \delta) / \epsilon_{\mathrm{est}}})  \times |\states| \times |\actions| \times |\actions|^{2|\states|} \\
\le &  
2^{66} \cdot (|\states| + 1)^{24(|\states| + 1)} \cdot |\actions|^{2|\states|} \cdot \log(18 |\states|^2|\actions| / \delta) \cdot |\states|^8 |\actions|^6 / \varepsilon^5.
\end{align*}
trajectories. 
\end{proof}
Theorem~\ref{thm:rl_oracle} immediately implies the following corollary. 
\begin{corollary}\label{corollary:rl_main}
Algorithm~\ref{alg:planning} returns a policy $\pi$ such that
\[
V^{\pi}_{M, H} \ge V^{\pi^*}_{M, H} - \epsilon
\]
with probability at least $1 - \delta$.
Moreover, the algorithm samples at most 
\[
2^{66} \cdot (|\states| + 1)^{24(|\states| + 1)} \cdot |\actions|^{2|\states|} \cdot \log(12 |\states|^2|\actions| / \delta) \cdot |\states|^8 |\actions|^6 / \varepsilon^5.
\]
trajectories. 
\end{corollary}

%% file: gen_model.tex
\section{Algorithm in the Generative Model}
In this section, we present our algorithm in the generative model together with its analysis.
The formal description of the algorithm is given in Algorithm~\ref{alg:gen_model}.
Compared to our algorithm in the RL setting, Algorithm~\ref{alg:gen_model} is conceptually much simpler.
For each state-action pair $(s, a) \in \states \times \actions$, Algorithm~\ref{alg:gen_model} first draws $N$ samples from $P(s, a)$ and $R(s, a)$, and then builds a model $\widehat{M}$ using the empirical estimators. 
The algorithm simply returns the optimal policy for $\widehat{M}$.
In the remaining part of this section, we present the formal analysis for our algorithm.

\begin{algorithm}
  \caption{Algorithm in Generative Model}\label{alg:gen_model}
  \begin{algorithmic}[1]
  \State \textbf{Input:} desired accuracy $\epsilon$, failure probability $\delta$
  \State \textbf{Output:} An $\epsilon$-optimal policy $\pi$
  \State Let $N = 2^{29} \cdot |\states|^5  \cdot |\actions|^3  \cdot H/ \epsilon ^3$
  \For{$(s, a) \in \states \times \actions$}
  \State Draw $s'_0, s'_1, \ldots, s'_{N - 1}$ from $P(s, a)$
  \State Draw $r_0, r_1, \ldots, r_{N - 1}$ from $R(s, a)$
  \State For each $s' \in \states$, let $\widehat{P}(s' \mid s, a) = \sum_{i = 0}^{N - 1} \indict[s'_i = s] / N$
  \State Let $\widehat{R}(s, a) = \sum_{i = 0}^{N - 1}r_i / N$
  \EndFor
  \State Draw $s_0, s_1, \ldots, s_{N - 1}$ from $\mu$
  \State For each $s'\in \states$, let $\widehat{\mu}(s ) = \sum_{i = 0}^{N - 1} \indict[s_i = s] / N$
  \State \Return{$\mathrm{argmax}_\pi V^{\pi}_{\widehat{M}, H}$ where $\widehat{M} = (\states, \actions, \widehat{P}, \widehat{R}, H, \widehat{\mu})$}
  \end{algorithmic}
\end{algorithm}

The following lemma establishes a perturbation analysis similar to those in Lemma~\ref{lem:perturbation_lower_bound} and Lemma~\ref{lem:perturbation_upper_bound}. 
Note that here, due to the availability of a generative model and thus one can obtain $\Omega(H)$ samples for each state-action pair $(s, a) \in \states \times \actions$, we can prove both upper bound and lower bound on the estimated value. 
This also explains why pessimistic planning is no longer necessary in the generative model setting. 
\begin{lemma}\label{lem:perturbation_gen_model}
Let $M=(\states,\actions, P, R, H, \mu)$ be an MDP and $\pi$ be a policy.
Let $0 < \epsilon \le 1/2$ be a parameter.
Let $\widehat{M}=(\states,\actions, \widehat{P}, \widehat{R}, H, \widehat{\mu})$ be another MDP.
If for all $(s,a, s') \in \states \times \actions \times \states$, we have

    \[|\widehat{P}(s'|s,a) - P(s'|s,a)|\le \frac{\epsilon}{192|\cS|^2|\cA|}\cdot \max\left( \sqrt{\frac{\epsilon{P}(s'|s,a)}{576\cdot H\cdot |\states||\actions|}}, \frac{\epsilon}{576\cdot H\cdot |\states||\actions|}\right),\]
 \[
    \left|\expect[\widehat{R}|s,a] - 
    \expect[{R}|s,a] \right| \le \frac{\epsilon}{48|\cS||\cA|}\cdot \max \left \{ \sqrt{\frac{\expect[(R(s, a))^2]}{H}} , \frac{1}{H}\right \},
    \]
    and
    \[\left|\mu(s) - \widehat{\mu}(s')\right| \le \epsilon / (12|\states|),\]
then
\[
\left|V^{\pi}_{\widehat{M},H} - V^{\pi}_{{M}, H}  \right| \le \epsilon / 2. 
\]
\end{lemma}
\begin{proof}
Define $\mathcal{T} = (\mathcal{S} \times \mathcal{A})^H \times \mathcal{S}$ be set of all possible trajectories, where for each $T \in \mathcal{T}$, $T$ has the form
\[
((s_0, a_0), (s_1, a_1), \ldots, (s_{H - 1}, a_{H - 1}), s_H). 
\]
For a trajectory $T\in \cT$, for each $(s, a, s') \in \mathcal{S} \times \mathcal{A} \times \mathcal{S}$, we write
\[m_T(s,a)=\sum_{h=0}^{H - 1} \indict((s_h,a_h)=(s,a))\]
as the number times $(s,a)$ is visited and 
\[
m_T(s,a,s')=\sum_{h= 0}^{H - 1}\indict((s_h,a_h, s_{h+1})=(s,a, s'))
\]
as the number of times $(s,a,s')$ is visited. 
We say a trajectory \[T = ((s_0, a_0), (s_1, a_1), \ldots, (s_{H - 1}, a_{H - 1}), s_H) \in \mathcal{T}\]is {\em compatible} with a (possibly non-stationary) policy $\pi$ if for all $h \in [H]$,
\[
a_h = \pi_h(s_h).
\]
For a (possibly non-stationary) policy $\pi$, we use $\mathcal{T}^{\pi} \subseteq \mathcal{T}$ to denote the set of all trajectories that are compatible with $\pi$.

For an MDP $M=(\states,\actions, P, R, H, \mu)$ and a (possibly non-stationary) policy $\pi$, for a trajectory $T$ that is compatible with $\pi$, we write
\[
p(T, M, \pi) = \mu(s_0) \cdot \prod_{h = 0}^{H - 1} P(s_{h + 1} \mid s_h, a_h) = \mu(s_0) \cdot \prod_{(s, a, s') \in \mathcal{S} \times \mathcal{A} \times \mathcal{S}} P(s' \mid s, a)^{m_T(s, a, s')} 
\]
to be the probability of $T$ when executing $\pi$ in $M$. Here we assume $0^0 = 1$.

Using these definitions, we have
\[
V^{\pi}_{M, H} = \sum_{T \in \mathcal{T}^{\pi}} p(T, M, \pi) \cdot \left( \sum_{(s, a) \in \states \times \actions} m_T(s, a) \cdot \expect[R(s, a)]\right).
\]

Note that for any trajectory $T = ((s_0, a_0), (s_1, a_1), \ldots, (s_{H - 1}, a_{H - 1}), s_H) \in \mathcal{T}^{\pi}$, if $p(T, M, \pi) > 0$, by Assumption~\ref{assump:reward}, 
\[
\sum_{(s, a) \in \states \times \actions} m_T(s, a) \cdot \expect[R(s, a)] \le 1.
\]



We define $\cT^{\pi}_1 \subseteq \cT^{\pi}$ be the set of trajectories such that for each $T \in  \cT^{\pi}_1$, for each $(s, a, s') \in \states \times \actions \times \states$, if $m_T(s,a, s') \ge 1$ then
\[
P(s'|s,a) \ge \frac{\epsilon}{144|\states||\actions|H} .
\]
By Lemma~\ref{lem:upperbound_on_visits} and a union bound over all $(s, a) \in \states \times \actions$, we have
\[
\sum_{T \in \mathcal{T}^{\pi} \setminus  \mathcal{T}^{\pi}_1} p(T, M, \pi) \le \epsilon / 12. 
\]
We also define  $\widehat{\cT}^{\pi}_1 \subseteq \cT^{\pi}$ be the set of trajectories such that for each $T \in  \widehat{\cT}^{\pi}_1$, for each $(s, a, s') \in \states \times \actions \times \states$, if $m_T(s,a, s') \ge 1$ then
\[
 \widehat{P}(s'|s,a) \ge \frac{\epsilon}{72|\states||\actions|H}.
\]
Again by Lemma~\ref{lem:upperbound_on_visits} and a union bound over all $(s, a) \in \states \times \actions$, we have
\[
\sum_{T \in \mathcal{T}^{\pi} \setminus  \widehat{\mathcal{T}}^{\pi}_1} p(T, \widehat{M}, \pi) \le \epsilon / 12. 
\]
Now we show that 
\[
\mathcal{T}^{\pi} \setminus  \mathcal{T}^{\pi}_1 \subseteq \mathcal{T}^{\pi} \setminus  \widehat{\mathcal{T}}^{\pi}_1,
\]
and therefore,
\[
\sum_{T \in \mathcal{T}^{\pi} \setminus  \mathcal{T}^{\pi}_1} p(T, \widehat{M}, \pi) \le \epsilon / 12. 
\]

For each $T \in \mathcal{T}^{\pi} \setminus  \mathcal{T}^{\pi}_1$, if there exists $(s, a, s') \in \states \times \actions \times \states$ such that $m_T(s, a, s') \ge 1$ and \[P(s' \mid s, a) < \frac{\epsilon}{144 |\states||\actions|H},\] then 
\[
\widehat{P}(s' \mid s, a) \le P(s' \mid s, a)  +  \frac{\epsilon}{192|\cS|^2|\cA|}\cdot \max\left( \sqrt{\frac{\epsilon{P}(s'|s,a)}{576\cdot H\cdot |\states||\actions|}}, \frac{\epsilon}{576\cdot H\cdot |\states||\actions|}\right) < \frac{\epsilon}{72 |\states||\actions|H}
\]
and thus $T \in  \mathcal{T}^{\pi} \setminus  \widehat{\mathcal{T}}^{\pi}_1$.
Moreover, for any $T \in \mathcal{T}^{\pi} \setminus  \mathcal{T}^{\pi}_1$, for all $(s, a, s') \in \states \times \actions \times \states$ with $m_T(s, a, s') \ge 1$, we have 
\[
P(s'|s,a) \ge \frac{\epsilon}{144|\states||\actions|H} > 0.
\]
Note that this implies $p(T, M, \pi) > 0$.
Furthermore, for all $(s, a, s') \in \states \times \actions \times \states$ with $m_T(s, a, s') \ge 1$, we have 
\[
\widehat{P}(s' \mid s, a) \ge P(s' \mid s, a)  -  \frac{\epsilon}{192|\cS|^2|\cA|}\cdot \max\left( \sqrt{\frac{\epsilon{P}(s'|s,a)}{576\cdot H\cdot |\states||\actions|}}, \frac{\epsilon}{576\cdot H\cdot |\states||\actions|}\right) > 0
\]
which implies $p(T, \widehat{M}, \pi)>0$.

We define $\cT^{\pi}_2 \subseteq \cT^{\pi}_1$ be the set of trajectories that for each $T \in  \cT^{\pi}_2$, for each $(s, a, s') \in \states \times \actions \times \states$,
\[
\left| m_T(s,a,s') - m_T(s,a) \cdot P(s'|s,a) \right| \le \sqrt{\frac{576P(s'|s,a) H |\states||\actions|}{\epsilon}}.
\]
By Lemma~\ref{lemma:martingal_concent} and a union bound over all $(s, a) \in \states\times \actions$, we have 
\[
\sum_{T \in \mathcal{T}^{\pi}_1 \setminus  \mathcal{T}^{\pi}_2} p(T, M, \pi)  \le \sum_{T \in \mathcal{T}^{\pi} \setminus  \mathcal{T}^{\pi}_2} p(T, M, \pi) \le \epsilon / 12. 
\]
Again, we define $\widehat{\cT}^{\pi}_2 \subseteq \cT^{\pi}$ be the set of trajectories that for each $T \in  \cT^{\pi}_2$, for each $(s, a, s') \in \states \times \actions \times \states$,
\[
\left| m_T(s,a,s') - m_T(s,a) \cdot \widehat{P}(s'|s,a) \right| \le \sqrt{\frac{144\widehat{P}(s'|s,a) H |\states||\actions|}{\epsilon}}.
\]
Similarly, by Lemma~\ref{lemma:martingal_concent} and a union bound over all $(s, a) \in \states\times \actions$, we have 
\[
\sum_{T \in \mathcal{T}^{\pi}_1 \setminus  \widehat{\mathcal{T}}^{\pi}_2} p(T, \widehat{M}, \pi) \le \sum_{T \in \mathcal{T}^{\pi} \setminus  \widehat{\mathcal{T}}^{\pi}_2} p(T, \widehat{M}, \pi) \le \epsilon / 12. 
\]
Now we show that 
\[
\mathcal{T}^{\pi}_1 \setminus  \mathcal{T}^{\pi}_2 \subseteq \mathcal{T}^{\pi}_1 \setminus  \widehat{\mathcal{T}}^{\pi}_2,
\]
and therefore,
\[
\sum_{T \in \mathcal{T}^{\pi}_1 \setminus  \mathcal{T}^{\pi}_2} p(T, \widehat{M}, \pi) \le \epsilon / 12. 
\]
For any $T \in \mathcal{T}^{\pi}_1 \setminus  \mathcal{T}^{\pi}_2$, there exists $(s, a, s') \in \states \times \actions \times \states$ such that
\[
\left| m_T(s,a,s') - m_T(s,a) \cdot P(s'|s,a) \right| > \sqrt{\frac{576P(s'|s,a) H |\states||\actions|}{\epsilon}}.
\]
Note that this implies
\[
P(s'|s,a) \ge \frac{\epsilon}{144 |\states||\actions|H}
\]
since otherwise $m_T(s, a, s') = 0$ (because $T \in \mathcal{T}^{\pi}_1$) and therefore
\[
\left| m_T(s,a,s') - m_T(s,a) \cdot P(s'|s,a) \right|  \le |H \cdot P(s'|s,a)| \le \sqrt{\frac{576P(s'|s,a) H |\states||\actions|}{\epsilon}}.
\]
Since 
\[
P(s'|s,a) \ge \frac{\epsilon}{144 |\states||\actions|H},
\]
we have
\[
\widehat{P}(s' \mid s, a) \le 2P(s' \mid s, a)
\]
and
\[
\left|\widehat{P}(s' \mid s, a) - P(s' \mid s, a) \right| \le \sqrt{\frac{\epsilon \cdot P(s' \mid s, a)}{144 |\states||\actions|H}}.
\]
Hence,
\begin{align*}
\left| m_T(s,a,s') - m_T(s,a) \cdot \widehat{P}(s'|s,a) \right| &\ge \left| m_T(s,a,s') - m_T(s,a) \cdot P(s'|s,a) \right|  - H \cdot   \sqrt{\frac{\epsilon \cdot P(s' \mid s, a)}{144 |\states||\actions|H}}\\
& > \sqrt{\frac{576P(s'|s,a) H |\states||\actions|}{\epsilon}} - \sqrt{\frac{\epsilon \cdot P(s' \mid s, a) \cdot H}{144 |\states||\actions|}}\\
& \ge \sqrt{\frac{500P(s'|s,a) H |\states||\actions|}{\epsilon}} \\
& \ge  \sqrt{\frac{144\widehat{P}(s'|s,a) H |\states||\actions|}{\epsilon}}
\end{align*}
and thus $T \in \mathcal{T}^{\pi}_1 \setminus  \widehat{\mathcal{T}}^{\pi}_2$.

%
%
%
%

Note that
\[
V^{\pi}_{M, H}  - \epsilon / 6 \le  \sum_{T \in \mathcal{T}^{\pi}_2} p(T, M, \pi) \cdot \left( \sum_{(s, a) \in \states \times \actions} m_T(s, a) \cdot \expect[R(s, a)]\right) \le V^{\pi}_{M, H}.
\]
Similarly,
\[
V^{\pi}_{\widehat{M}, H}  - \epsilon / 6 \le \sum_{T \in \mathcal{T}^{\pi}_2} p(T, \widehat{M}, \pi)\cdot \left( \sum_{(s, a) \in \states \times \actions} m_T(s, a) \cdot \expect[\widehat{R}(s, a)]\right) \le V^{\pi}_{\widehat{M}, H}  .
\]

For each $(s, a) \in \states \times \actions$, define 
\[
\states_{s,a}=
\left\{
s'\in \states \
\mid P(s'|s,a)\ge \frac{\epsilon}{72|\cS||\cA|H}
\right\}. 
\]
Therefore, for each $(s, a) \in \states \times \actions$, 
\[
\prod_{s' \in \mathcal{S}} P(s' \mid s, a)^{m_T(s, a, s')}  = \prod_{s'\in \states_{s,a}}P(s' \mid s, a)^{m_T(s, a, s')} 
\]
and
\[
\prod_{s' \in \mathcal{S}} \widehat{P}(s' \mid s, a)^{m_T(s, a, s')}  = \prod_{s'\in \states_{s,a}} \widehat{P}(s' \mid s, a)^{m_T(s, a, s')} 
\]
with  
\begin{align*}
\left| m_T(s,a,s') - m_T(s,a) \cdot P(s'|s,a) \right| \le  \sqrt{\frac{576P(s'|s,a)H|\states||\actions|}{\epsilon}}.
\end{align*}
Note that $\sum_{s'\in \states}\widehat{P}(s'|s,a) - {P}(s'|s,a)
=0$, which implies 
\[
\left|\sum_{s'\in \states_{s,a}}\widehat{P}(s'|s,a) - {P}(s'|s,a)\right|
=
\left|\sum_{s'\not\in \states_{s,a}}{P}(s'|s,a) - \widehat{P}(s'|s,a)\right|
\le \frac{\epsilon}{192|\states||\actions|}\cdot \frac{\epsilon}{144 \cdot H\cdot |\states||\actions|}.
\]

By applying Lemma~\ref{lem:mult-add}
and settting $\bar{n}$ to be $|\states|$, $n$ to be $|\states_{s,a}|$, $\epsilon$ to be $ \epsilon/(192|\states|^2|\actions|)$, and $\mb$ to be $576 \cdot H\cdot |\states||\actions| / \epsilon$,
we have
\[
 \prod_{s'\in \states_{s,a}} \widehat{P}(s' \mid s, a)^{m_T(s, a, s')} 
\ge \left(1 - \frac{\epsilon}{24|\states||\actions|}\right)  \prod_{s'\in \states_{s,a}} P(s' \mid s, a)^{m_T(s, a, s')}
\]
and
\[
 \prod_{s'\in \states_{s,a}} \widehat{P}(s' \mid s, a)^{m_T(s, a, s')} 
\le \left(1 + \frac{\epsilon}{24|\states||\actions|}\right)  \prod_{s'\in \states_{s,a}} P(s' \mid s, a)^{m_T(s, a, s')} .
\]
Therefore,
\[
\prod_{(s, a) \in \states \times \actions}  \prod_{s'\in \states} \widehat{P}(s' \mid s, a)^{m_T(s, a, s')} 
\ge (1 - \epsilon / 12) \prod_{(s, a) \in \states \times \actions}  \prod_{s'\in \states} P(s' \mid s, a)^{m_T(s, a, s')}.
\]
and
\[
\prod_{(s, a) \in \states \times \actions}  \prod_{s'\in \states} \widehat{P}(s' \mid s, a)^{m_T(s, a, s')} 
\le (1 + \epsilon / 12) \prod_{(s, a) \in \states \times \actions}  \prod_{s'\in \states} P(s' \mid s, a)^{m_T(s, a, s')}.
\]
For the summation of rewards, we have
\begin{align*}
&\sum_{(s,a)\in \states \times \actions} m_T(s, a) \cdot \left| \expect\left[R(s, a)\right] - \expect\left[\widehat{R}(s, a)\right] \right|\\
= &\sum_{(s, a) \in \states \times \actions \mid m_T(s,a)=1} m_T(s, a) \cdot \left| \expect\left[R(s, a)\right] - \expect\left[\widehat{R}(s, a)\right] \right|\\
+ &\sum_{(s, a) \in \states \times \actions \mid m_T(s,a)>1} m_T(s, a) \cdot \left| \expect\left[R(s, a)\right] - \expect\left[\widehat{R}(s, a)\right] \right|.
\end{align*}
For those $(s, a) \in \states \times \actions$ with $m_T(s,a)>1$, since $p(T, M, \pi) > 0$, by Lemma~\ref{lem:reward}, we have
\begin{align*}
&|\expect[\widehat{R}(s, a)] - \expect[R(s, a)]|\\
\le&  \frac{\epsilon}{48|\cS||\cA|}\cdot \max\left\{ \sqrt{\frac{\expect[(R(s, a))^2]}{H}}, \frac{1}{H} \right\}\le 
\max\left\{\frac{\epsilon}{24H}, \frac{\epsilon}{48 |\states| |\actions| H}\right\}.
\end{align*}
Since $\sum_{(s, a)\in\states\times\actions} m_T(s,a)\le H$ and $m_T(s, a) \le H$, we have
\[
\sum_{(s, a) \in \states \times \actions \mid m_T(s,a)>1} m_T(s, a) \cdot \left| \expect\left[R(s, a)\right] - \expect\left[\widehat{R}(s, a)\right] \right|
\le 
\frac{\epsilon}{16}.
\]
For those $(s, a) \in \states \times \actions$ with $m_T(s,a)= 1$, we have
\[
\sum_{(s, a) \in \states \times \actions \mid m_T(s,a)=1} m_T(s, a) \cdot \left| \expect\left[R(s, a)\right] - \expect\left[\widehat{R}(s, a)\right] \right| \le \frac{\epsilon}{48}.
\]
Thus,
\[
\sum_{(s,a)\in \states \times \actions} m_T(s, a) \cdot \left| \expect\left[R(s, a)\right] - \expect\left[\widehat{R}(s, a)\right] \right| \le \frac{\epsilon}{12}.
\]

For each $T \in \mathcal{T}^{\pi}_2$ with
\[
T = ((s_0, a_0), (s_1, a_1), \ldots, (s_{H - 1}, a_{H - 1}), s_H), 
\]
we have
\begin{align*}
& p(T, \widehat{M}, \pi)\cdot \left( \sum_{(s, a) \in \states \times \actions} m_T(s, a) \cdot \expect[\widehat{R}(s, a)]\right) \\
= & \widehat{\mu}(s_0)  \cdot \prod_{(s, a) \in \states \times \actions} \prod_{s' \in \states} \widehat{P}(s' \mid s, a)^{m_T(s, a, s')} \cdot \left( \sum_{(s, a) \in \states \times \actions} m_T(s, a) \cdot \expect[\widehat{R}(s, a)]\right)\\
\end{align*}
where
\[
\left| \widehat{\mu}(s_0) - \mu(s_0) \right| \le \epsilon / (12|\states|), 
\]
\[
\left| \sum_{(s, a) \in \states \times \actions} m_T(s, a) \cdot \expect[\widehat{R}(s, a)] - \sum_{(s, a) \in \states \times \actions} m_T(s, a) \cdot \expect[R(s, a)] \right| \le \epsilon / 12,
\]
\[
\prod_{(s, a) \in \states \times \actions}  \prod_{s'\in \states} \widehat{P}(s' \mid s, a)^{m_T(s, a, s')} 
\ge (1 - \epsilon / 12) \prod_{(s, a) \in \states \times \actions}  \prod_{s'\in \states} P(s' \mid s, a)^{m_T(s, a, s')},
\]
and
\[
\prod_{(s, a) \in \states \times \actions}  \prod_{s'\in \states} \widehat{P}(s' \mid s, a)^{m_T(s, a, s')} 
\le (1 + \epsilon / 12) \prod_{(s, a) \in \states \times \actions}  \prod_{s'\in \states} P(s' \mid s, a)^{m_T(s, a, s')}.
\]
Since
\[
V^{\pi}_{M, H}  - \epsilon / 6 \le  \sum_{T \in \mathcal{T}^{\pi}_2} p(T, M, \pi) \cdot \left( \sum_{(s, a) \in \states \times \actions} m_T(s, a) \cdot \expect[R(s, a)]\right) \le V^{\pi}_{M, H}
\]
and
\[
V^{\pi}_{\widehat{M}, H}  - \epsilon / 6 \le \sum_{T \in \mathcal{T}^{\pi}_2} p(T, \widehat{M}, \pi)\cdot \left( \sum_{(s, a) \in \states \times \actions} m_T(s, a) \cdot \expect[\widehat{R}(s, a)]\right) \le V^{\pi}_{\widehat{M}, H},
\]
we have 
\[
\left| V^{\pi}_{\widehat{M}, H} - V^{\pi}_{M, H}\right| \le \epsilon / 2. 
\]

\end{proof}

The following lemma provides error bound on the empirical model $\widehat{M}$ built by the algorithm.
Note that the error bounds established here are similar to those in Lemma~\ref{lem:approximation}.
\begin{lemma}\label{lem:approximation_gen_model}
With probability at least $1- \delta$, for all $(s, a, s') \in \states \times \actions \times \states$:
\begin{align*}
\left| \widehat{P}(s' \mid s, a) - P(s' \mid s, a) \right| 
\le & \max\left\{ 4\sqrt{\frac{P(s' \mid s, a)\cdot \log(6|\states|^2|\actions| / \delta)}{N} }, \frac{2\log(6|\states|^2|\actions| / \delta)}{N}  \right\},
\end{align*}
\begin{align*}
\left| \widehat{R}(s, a) - \expect[R(s, a)] \right| \le & \max \left\{ 4\sqrt{\frac{\expect[(R(s, a))^2]\cdot \log(6|\states||\actions| / \delta)}{N}} , \frac{2\log(6|\states||\actions| / \delta)}{N}\right\}
\end{align*}
and 
\[
\left| \widehat{\mu}(s) -\mu(s) \right| \le \sqrt{ \frac{\log((6|\states|) / \delta)}{N}}.
\]
\end{lemma}
\begin{proof}
For each $(s, a, s') \in \states \times \actions \times \states$, by Bernstein's inequality, 
\begin{align*}
& \Pr \left[ \left|
 \sum_{i = 0}^{N - 1}  \indict\left[s' = s'_i \right] - P(s' \mid s, a) \cdot N
 \right|  \ge t \right]  \\
 \le & 2\exp\left( \frac{-t^2}{2 \cdot N \cdot P(s' \mid s, a) + t / 3} \right).
\end{align*}
Thus, by setting $t = 2\sqrt{N \cdot P(s' \mid s, a)\cdot \log(6|\states|^2|\actions| / \delta)} + \log(6|\states|^2|\actions| / \delta)$, we have
\[
Pr \left[ \left|
 \sum_{i = 0}^{N - 1}  \indict\left[s' = s'_i \right] - P(s' \mid s, a) \cdot N
 \right|  \ge t \right]  \le \delta / (3|\states|^2|\actions|).
\]
By applying a union bound over all $(s, a, s') \in \states \times \actions \times \states$, with probability at least $1 - \delta / 3$, for all $(s, a, s') \in \states \times \actions \times \states$,
\begin{align*}
\left| \widehat{P}(s' \mid s, a) - P(s' \mid s, a) \right| 
\le &2\sqrt{\frac{P(s' \mid s, a)\cdot \log(6|\states|^2|\actions| / \delta)}{N} } + \frac{\log(6|\states|^2|\actions| / \delta)}{N}\\
\le & \max\left\{ 4\sqrt{\frac{P(s' \mid s, a)\cdot \log(6|\states|^2|\actions| / \delta)}{N} }, \frac{2\log(6|\states|^2|\actions| / \delta)}{N} \right\}.
\end{align*}

By a similar argument, with probability at least $1 - \delta / 3$, for all $(s, a) \in \states \times \actions$,
\[
\left|\widehat{R}(s, a) - \expect[R(s, a)] \right| \le \max \left\{ 4\sqrt{\frac{\expect[(R(s, a))^2]\cdot \log(6|\states||\actions| / \delta)}{N}} , \frac{2\log(6|\states||\actions| / \delta)}{N}\right\}. 
\]
Moreover, with probability at least $1 - \delta / 3$, for all $s \in \states$, we have
\[
\left| \widehat{\mu}(s) -\mu(s) \right| \le \sqrt{ \frac{\log(6|\states|) / \delta)}{N}}.
\]
We complete the proof by applying a union bound. 
\end{proof}

Using Lemma~\ref{lem:perturbation_gen_model} and Lemma~\ref{lem:approximation_gen_model}, we can now prove the correctness of our algorithm.
\begin{theorem}\label{thm:gen_model_main}
With probability at least $1 - \delta$, Algorithm~\ref{alg:planning} returns a policy $\pi$ such that
\[
V^{\pi} \ge V^* - \epsilon
\]
by using at most $2^{30} \cdot |\states|^6 \cdot |\actions|^4  / \epsilon^3$ batches of queries. 
\end{theorem}
\begin{proof}
By Lemma~\ref{lem:approximation_gen_model}, with probability at least $1 - \delta$, 
for all $(s,a, s') \in \states \times \actions \times \states$, we have

    \[|\widehat{P}(s'|s,a) - P(s'|s,a)|\le \frac{\epsilon}{192|\cS|^2|\cA|}\cdot \max\left( \sqrt{\frac{\epsilon{P}(s'|s,a)}{576\cdot H\cdot |\states||\actions|}}, \frac{\epsilon}{576\cdot H\cdot |\states||\actions|}\right),\]
 \[
    \left|\expect[\widehat{R}|s,a] - 
    \expect[{R}|s,a] \right| \le \frac{\epsilon}{48|\cS||\cA|}\cdot \max \left \{ \sqrt{\frac{\expect[(R(s, a))^2]}{H}} , \frac{1}{H}\right \},
    \]
    and
    \[\left|\mu(s) - \widehat{\mu}(s')\right| \le \epsilon / (12|\states|).\]
    We condition on this event in the remaining part of this proof.
    By Lemma~\ref{lem:approximation_gen_model}, for any (possibly non-stationary) policy $\pi$, we have
    \[
\left|V^{\pi}_{\widehat{M},H} - V^{\pi}_{{M}, H}  \right| \le \epsilon / 2. 
\]
Let $\pi$ be the policy retuned by the algorithm. We thus have
\[
V^{\pi}_{M, H} \ge V^{\pi}_{\widehat{M}, H} - \epsilon / 2 \ge V^{\pi^*}_{\widehat{M}, H} - \epsilon / 2\ge V^{\pi^*}_{M, H} - \epsilon. 
\]

The total number of samples taken by the algorithm is 
\[
(|\states||\actions| + 1) \cdot N \le 2^{30} \cdot |\states|^6 \cdot |\actions|^4  \cdot H / \epsilon^3.
\]

\end{proof}